\documentclass[preprint]{elsarticle}
\makeatletter
\def\ps@pprintTitle{%
 \let\@oddhead\@empty
 \let\@evenhead\@empty
 \def\@oddfoot{}%
 \let\@evenfoot\@oddfoot}
\makeatother

\usepackage{float}
\usepackage{amsmath}
\usepackage{amsthm}
\usepackage{tikz}
\usepackage{float}
\usepackage{subcaption}
\usepackage{rotating}
\usepackage{tabularx}
\usepackage{color}
\usepackage{fancyvrb}

\usetikzlibrary{positioning, arrows.meta, shapes.geometric}
\tikzset{%
  dot/.style n args = {4}{name=#3, circle, draw, inner sep=1pt, minimum size=4pt, fill=black, label={[shift={(#1,#2)}]#4:$#3$}},
  lat/.style n args = {4}{name=#3, circle, draw, inner sep=1pt, minimum size=4pt, label={[shift={(#1,#2)}]#4:$#3$}},
  tr/.style  n args = {4}{name=#3, regular polygon,regular polygon sides=4, draw, inner sep=1pt, minimum size=5pt, fill=gray, label={[shift={(#1,#2)}]#4:$#3$}},
  intv/.style n args = {4}{name=#3, regular polygon,regular polygon sides=4, draw, inner sep=1pt, minimum size=5pt, fill=black, label={[shift={(#1,#2)}]#4:$#3$}},
  >={Latex[width=1.5mm,length=2mm]},
  every picture/.style={semithick}
}

\usepackage{algorithm}
\usepackage{algorithmicx}
\usepackage[noend]{algpseudocode}

\algnewcommand\algorithmicinput{\textbf{Input:}}
\algnewcommand\INPUT{\item[\algorithmicinput]}
\algnewcommand\algorithmicoutput{\textbf{Output:}}
\algnewcommand\OUTPUT{\item[\algorithmicoutput]}
\algrenewcommand\alglinenumber[1]{\footnotesize\bfseries#1:}

\def\P{P}

\newcommand{\doo}{\textrm{do}}
\newcommand{\cond}{\,\vert\,}
\newcommand{\Pa}{\textrm{Pa}}
\newcommand{\Ch}{\textrm{Ch}}
\newcommand{\An}{\textrm{An}}
\newcommand{\De}{\textrm{De}}
\def\independenT#1#2{\mathrel{\rlap{$#1#2$}\mkern2mu{#1#2}}}
\newcommand{\indep}{\protect\mathpalette{\protect\independenT}{\perp}}

\newtheorem{theorem}{Theorem}[section]
\newtheorem{definition}{Definition}[section]
\newtheorem{corollary}{Corollary}[section]
\newtheorem{lemma}{Lemma}[section]
\newproof{pol}{Proof of Lemma \ref{lem:dist_property}}

\begin{document}

\title{Surrogate Outcomes and Transportability}

\author{S.~Tikka\corref{cor1}}
\ead{santtu.tikka@jyu.fi}

\author{J.~Karvanen\corref{cor2}}
\ead{juha.t.karvanen@jyu.fi}
\address{Department of Mathematics and Statistics, University of Jyvaskyla, P.O. Box 35 (MaD) FI-40014, Finland}

\cortext[cor1]{Corresponding author.}

\begin{abstract}
Identification of causal effects is one of the most fundamental tasks of causal inference. We consider an identifiability problem where some experimental and observational data are available but neither data alone is sufficient for the identification of the causal effect of interest. Instead of the outcome of interest, surrogate outcomes are measured in the experiments. This problem is a generalization of identifiability using surrogate experiments \citep{bareinboim2012a} and we label it as surrogate outcome identifiability. We show that the concept of transportability \citep{bareinboim2013:general} provides a sufficient criteria for determining surrogate outcome identifiability for a large class of queries.
\end{abstract}


\begin{keyword}
Causality \sep do-calculus \sep Experiment \sep Graph \sep Identifiability \sep Mediator.
\end{keyword}

\maketitle

\section{Introduction} \label{sect:intro}

In the formal framework of causal inference it is sometimes possible to make experimental claims using observational data alone. First, we construct a causal model by encoding our knowledge into a graph and specify a probability distribution over the observed variables. An experiment can now be carried out symbolically in the model through an intervention, which is an action that forces variables to take specific values irrespective of the mechanism that would determine their values otherwise. The question is whether the observed probability distribution alone is enough to determine the effect of this intervention. This problem, known as the \emph{identifiability} problem, has been studied extensively in literature and solutions in the form of graphical criteria \citep{pearl1995,pearl2009} as well as algorithms have been proposed \citep{huang2006:complete, shpitser2006, tian2002:phdthesis}. Various extensions to the identifiability problem have emerged in recent years. These include concepts such as \emph{transportability}, where identifiability is considered in a target population, but information for the task is available from multiple source populations \citep{bareinboim2013:metatransportability,pearl2014}. 

The presence of unobserved confounders often renders causal effects of interest non-identifiable from observational data alone. This leads us to ask whether experimental data can be of use in the identification task. The concept of \emph{surrogate experiments} or \emph{z-identifiability} considers this problem in a setting where in addition to the observed probability distribution, experimentation is allowed on a set of variables that is disjoint from the interventions of the target causal effect \citep{bareinboim2012a} and the experimental distribution of these surrogate experiments is available over all variables. By experimental distribution we mean a distribution of a set of outcomes variables when some variables have been intervened on. We consider a more general problem than $z$-identifiability: instead of assuming that a experimental distribution over all variables is available, we assume that a collection of experimental distributions is available where every variable has not necessarily been observed. This kind of setting can occur for example in mediation analysis, where we have previously performed an experiment where the mediator was the outcome variable. Another example is a setting where we are interested in two outcome variables but have only measured one of them in a previous experiment.

In a practical study we usually have access to information about population characteristics when performing an experiment. Sometimes not all of these characteristics are be measured in conjunction with the experiment itself which leads to incomplete knowledge regarding the experimental distribution. Suppose that we are interested in the experimental distribution of another variable, one that was not measured during the experiment. The question is whether this distribution can be obtained from the observational data and the outcome of the previous experiment, which we refer to as the \emph{surrogate outcome}. We label this generalization of identifiability as \emph{surrogate outcome identifiability}.

Remarkably, a connection can be drawn between surrogate outcome identifiability and transportability. Transportability is concerned with identifiability across conceptual \emph{domains} where both observational and experimental data are available from each domain. In practical terms, a domain can be for example a city, and data from multiple domains in this case could be for example the age distributions of the populations of these cities. Naturally, discrepancies between causal mechanisms can arise between domains, which has to be taken into account in the causal modeling framework. Typically, we are interested in the effect of an intervention in a single domain, known as the target domain, and the domains providing additional information for the task are known as source domains. However, existing methods for determining transportability only allow a single experiment to take place within a single domain, whereas our surrogate identifiability is concerned with multiple distributions from differing experiments in a single domain. We incorporate the framework of transportability by depicting each available experiment of the surrogate outcome problem as a source domain of a transportability problem with the same experiments.

An introductory example illustrates the difference between surrogate outcome identifiability and $z$-identifiability. We are interested in the causal effect of $X_1$ and $X_2$ on $Y_1$ and $Y_2$ in the graph of Fig.~\ref{fig:intro_example}, which is easily determined to be non-identifiable from the joint distribution $P(v)$ alone for example via the application of the ID algorithm \citep{shpitser2006,tikka2017:causaleffect}. Suppose now that two surrogate outcomes were measured in previous experiments providing us with two experimental distributions, $P(y_2 \cond \doo(x_2), x_1, z, w)$ and $P(y_1 \cond \doo(x_1), z, w)$. The availability of these two distributions cannot be represented as $z$-identifiability problem, since they are conditional causal effects and they have common interventions with the target causal effect. We cannot directly regard this problem as a transportability problem either, since we are concerned with only a single domain. The causal effect can now be identified with the help of the two experimental distributions, which we will show later in Section~\ref{sect:sotr}.

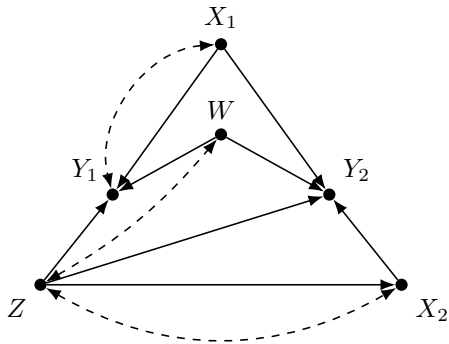
\begin{figure}[H]
  \centering
  \begin{tikzpicture}[scale=1.6]
  \node [dot = {0}{0}{X_1}{above}] at (1.5,2) {};
  \node [dot = {0}{0}{X_2}{below right}] at (3,0) {};
  \node [dot = {0}{0}{Z}{below left}] at (0,0) {};
  \node [dot = {0}{0}{Y_1}{above left}] at (0.6,0.75) {};
  \node [dot = {0}{0}{Y_2}{above right}] at (2.4,0.75) {};
  \node [dot = {0}{0}{W}{above}] at (1.5,1.25) {};
  
  \draw [->] (Z) -- (X_2);
  \draw [->] (Z) -- (Y_2);
  \draw [->] (Z) -- (Y_1);
  \draw [->] (W) -- (Y_1);
  \draw [->] (W) -- (Y_2);
  \draw [->] (X_1) -- (Y_2);
  \draw [->] (X_1) -- (Y_1);
  \draw [->] (X_2) -- (Y_2);
  
  \draw [<->,dashed] (X_1) to [bend right=50]  (Y_1);
  \draw [<->,dashed] (Z) to [bend right=10]  (W);
  \draw [<->,dashed] (Z) to [bend right=30]  (X_2);
  \end{tikzpicture}
  \caption{A graph where the causal effect $P(y_1,y_2 \cond \doo(x_1,x_2))$ is not identifiable from $P(v)$ alone.}
  \label{fig:intro_example}
\end{figure}
\noindent



In this paper we propose a way to transform a surrogate outcome problem into a transportability problem. We show that the identifiability of the transformed problem is a sufficient condition for identifiability of the surrogate outcome problem. We derive an identifiability algorithm for surrogate outcome problems and implement it as a part of the R package \emph{causaleffect} \citep{rsoft,tikka2017:causaleffect}.


\section{Notation and definitions} \label{sect:defi}

We assume that the reader is familiar with graph theoretic concepts fundamental to causal inference and refer them to works such as \citep{koller2009}. We use capital letters to denote vertices and the respective variables and small letters to denote their values. We sometimes write singleton sets $\{X\}$ as $X$ for clarity. A directed graph with a vertex set $V$ and an edge set $E$ is denoted by $(V, E)$. For a graph $G = (V,E)$ and a set of vertices $W \subseteq V$ the sets $\Pa(W)_G, \Ch(W)_G, \An(W)_G$ and $\De(W)_G$ denote a set that contains $W$ in addition to its parents, children, ancestors and descendants in $G$, respectively. A subgraph of a graph $G = (V, E)$ induced by a set of vertices $ W \subset  V$ is denoted by $G[W]$. This subgraph retains all edges $V_i \rightarrow V_j$ of $G$ such that $V_i,V_j \in  W$. The graph obtained from $G$ by removing all incoming edges of $X$ and all outgoing edges of $Z$ is written as $G[\overline{X},\underline{Z}]$. A back-door path from $X$ to $Y$ is a path with an edge incoming to $X$ and $Y$. A topological ordering $\varphi$ of $G$ is an ordering of its vertices in which every node is smaller than its descendants in $G$. The set of vertices smaller than a vertex $V_i$ in $\varphi$ is denoted by $V_\varphi^{(i)}$. To facilitate analysis of identifiability and the generalization to surrogate outcomes, we must first define the probabilistic causal model \citep{pearl2009}.

\begin{definition}[Probabilistic causal model]
A \emph{probabilistic causal model} is a quadruple
\[ M = (U, V, F, \P(u)), \]
where $U$ is a set of unobserved (exogenous) variables that are determined by factors outside the model,
$V$ is a set $\{V_1,\ldots,V_n\}$ of observed (endogenous) variables that are determined by variables in $U \cup V$. $F$ is a set of functions $\{f_{V_1},\ldots,f_{V_n}\}$ such that each $f_{V_i}$ is a mapping from (the respective domains of) $U \cup (V \setminus \{V_i\})$ to $V_i$, and such that the entire set $F$ forms a mapping from $U$ to $ V$, and $\P(u)$ is a joint probability distribution of the variables in the set $U$.
\end{definition}

Each causal model induces a graph through the following construction: A vertex is added for each variable in $ U \cup  V$ and a directed edge from $V_i \in U \cup V$ into $V_j \in  V$ whenever $f_{V_j}$ is defined in terms of $V_i$. Conventionally, causal inference focuses on a sub-class of models with additional assumptions: each $U_i \in U$ appears in at most two functions of $F$, the variables in $U$ are mutually independent and the induced graph of the model is acyclic. Models that satisfy these additional assumptions are called \emph{semi-Markovian causal models}. The induced graph of a semi-Markovian model is called a \emph{semi-Markovian graph}. In semi-Markovian graphs every $U_i \in  U$ has at most two children. In semi-Markovian models it is common not to depict background variables in the induced graph explicitly. Unobserved variables $U_i \in U$ with exactly two children are not denoted as $V_j \leftarrow U_i \rightarrow V_k$ but as a bidirected edge $V_j \leftrightarrow V_k$ instead. Furthermore, unobserved variables with only one or no children are omitted entirely. We also adopt these abbreviations. For semi-Markovian graphs the sets $\Pa(\cdot)_G, \Ch(\cdot)_G, \An(\cdot)_G$ and $\De(\cdot)_G$ contain only observed vertices. Additionally, a subgraph $G[W]$ of a semi-Markovian graph $G$ retains any bidirected edges between vertices in $W$.

A graph induced by a probabilistic causal model also encodes conditional independences among the variables in the model through a concept known as d-separation. We use the definition in \citep{shpitser2008} which explicitly accounts for the presence of bidirected edges making it suitable for semi-Markovian graphs.

\begin{definition}[d-separation] A path $P$ in a semi-Markovian graph $G$ is said to be d-separated by a set $ Z$ if and only if either $P$ contains one of the following three patterns of edges: $I \rightarrow M \rightarrow J$, $I \leftrightarrow M \rightarrow J$ or $I \leftarrow M \rightarrow J$, such that $M \in  Z$, or $P$ contains one of the following three patterns of edges: $I \rightarrow M \leftarrow J$, $I \leftrightarrow M \leftarrow J$, $I \leftrightarrow M \leftrightarrow J$, such that $\De(M)_G \cap  Z = \emptyset$. Disjoint sets $ X$ and $ Y$ are said to be d-separated by $ Z$ in $G$ if every path from $ X$ to $ Y$ is d-separated by $ Z$ in $G$.
\end{definition}
Since we are dealing entirely with semi-Markovian graphs, we will henceforth refer to them simply as graphs. If no conditional independence statements other than those already encoded in the graph are implied by the distribution of the variables in the model, we say that the distribution is \emph{faithful} \citep{spirtes2000}.

A causal model allows us to manipulate the functional relationships encoded in the set $F$. An \emph{intervention} $\doo(x)$ on a model $M$ forces $X$ to take the specified value $x$. The intervention also creates a new sub-model, denoted by $M_{x}$, where the functions in $F$ that determine the value of $X$ have been replaced with constant functions. The \emph{interventional distribution} of a set of variables $Y$ in the model $M_{x}$ is denoted by $\P(y \cond \doo(x))$. This distribution is also known as the \emph{causal effect} of $X$ on $Y$. Three inference rules known as \emph{do-calculus} \citep{pearl1995} provide the means for manipulating interventional distributions.
\begin{enumerate}
\item{Insertion and deletion of observations: 
\[
  \P(y \cond \doo(x), z, w) = \P(y \cond \doo(x), w), \text{ if } (Y \indep Z\cond X,  W)_{G[\overline{ X}]}.
\]}
\item{Exchange of actions and observations:
\[
  \P(y \cond \doo(x, z), w) = \P(y \cond \doo(x), z, w), \text{ if } (Y \indep Z\cond X,  W)_{G[\overline{ X},\underline{Z}]}.
\]}
\item{Insertion and deletion of actions
\[
  \P(y \cond \doo(x, z), w) = \P(y \cond \doo(x), w), \text{ if } (Y \indep Z\cond X,  W)_{G[\overline{X},\overline{Z(W)}]},
\]
where $Z(W) = Z \setminus \An(W)_{G[\overline{X}]}.$}
\end{enumerate}

Regarding the identifiability problem, the goal is to transform $\P(y \cond \doo(x))$ into an expression that does not contain the $\doo$-operator using do-calculus. A causal effect that admits this transformation is called \emph{identifiable}, which is formally defined in e.g. \citep{shpitser2006}. Do-calculus has been shown to be complete with respect to the identifiability problem \citep{shpitser2006,huang2006:complete} as well as the transportability and $z$-identifiability problems \citep{bareinboim2013:general,bareinboim2012a}.


Special graphs known as \emph{c-components} (confounded components) are crucial for causal effect identification \citep{shpitser2006}.
\begin{definition}[c-component] \label{def:c-component} Let $G = (V,E)$ be a graph. A c-component $C = (V_C,E_C)$ (of $G$) is a subgraph of $G$ such that every pair of vertices in $C$ is connected via a bidirected path (a path consisting entirely of bidirected edges). A c-component $C$ is \emph{maximal} if there are no vertices in $V_C$ that are connected to $V \setminus V_C$ in $G$ via bidirected paths and $C$ is an induced subgraph of $G$.
\end{definition}
The joint distribution of a causal model admits the so-called c-component factorization with respect to the set of maximal c-components of the induced graph $G$ of the model, denoted by $C(G)$. Henceforth we will use the term c-component to refer to maximal c-components for brevity.

If in addition to the joint observed probability distribution $\P(v)$ experimentation is allowed on a set $Z$, the identifiability problem is known as $z$-identifiability \citep{bareinboim2012a}. The set $Z$ is known as the set of \emph{surrogate experiments}.
\begin{definition}[$z$-identifiability] \label{def:zid} Let $G = (V, E)$ be a graph and let $X$, $Y$ and $Z$ be disjoint sets of variables such that $X,Y,Z \subset V$. The causal effect of $X$ on $Y$ is said to be $z$-\emph{identifiable} from $\P$ in $G$ if $\P(y \cond \doo(x))$ is uniquely computable from $\P(v)$ together with the interventional distributions $\P(v \setminus z^\prime \cond \doo(z^\prime))$, for all $Z^\prime \subseteq Z$, in any model that induces $G$.
\end{definition}

As an example of a $z$-identifiable causal effect, we consider the identification of $\P(y \cond \doo(x))$ from $\P(x,y,z,w)$ and $\P(x,y,w \cond \doo(z))$ in the graph of Fig.~\ref{fig:zidex_graph}. This effect is not identifiable without the experimental distribution, which can be verified for example by using the ID algorithm of \citep{shpitser2006}.

\begin{figure}[ht]
  \centering
  \begin{tikzpicture}[scale=1.7]
  \node [dot = {0}{0}{W}{below left}] at (-1,0) {};
  \node [dot = {0}{0}{Z}{below left}] at (0,0) {};
  \node [dot = {0}{0}{X}{below}] at (1,0) {};
  \node [dot = {0}{0}{Y}{below right}] at (2,0) {};
  
  \draw [->] (W) -- (Z);
  \draw [->] (Z) -- (X);
  \draw [->] (X) -- (Y);
  
  \draw [<->,dashed] (W) to [bend left=50] (Z);
  \draw [<->,dashed] (Z) to [bend left=50] (X);
  \draw [<->,dashed] (Z) to [bend left=50] (Y);
  \draw [->] (W) to [bend right=35]  (Y);
  \end{tikzpicture}
  \caption{A graph where the causal effect of $X$ on $Y$ is $z$-identifiable using experiments on $Z$.}
  \label{fig:zidex_graph}
\end{figure}
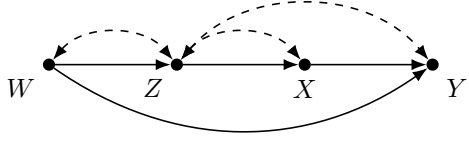
\noindent
We derive the effect using do-calculus:
\begin{align*}
\P(y \cond \doo(x))  &= \sum_{w} \P(y \cond \doo(x),w) \P(w \cond \doo(x)) \\
                     &= \sum_{w} \P(y \cond \doo(z,x),w) \P(w \cond \doo(x)) \\
                     &= \sum_{w} \P(y \cond \doo(z,x),w) \P(w) \\
                     &= \sum_{w} \P(y \cond \doo(z),x,w) \P(w)
\end{align*}
where the second equality follows from the third rule of do-calculus, since $(Y \indep Z \cond X)_{G[\overline{X},\overline{Z}]}$. The third equality follows from the third rule of do-calculus, since $(W \indep X)_{G[\overline{X}]}$ and the fourth equality follows from the second rule of do-calculus, since $(Y \indep X \cond W)_{G[\overline{Z},\underline{X}]}$. The term $\P(y \cond \doo(z), x)$ is identifiable from $\P(x,y,w \cond \doo(z))$ via marginalization and conditioning and $\P(w)$ is identifiable from $\P(x,y,z,w)$ via marginalization.


The available information in a $z$-identifiability problem consists of a single observational distribution and experimental distributions resulting from interventions on subsets of $Z$. Our goal is to extend this problem to a setting where experimentation is allowed on the subsets of multiple surrogate experiments. Furthermore, we do not require that the distribution of the entire set $V$ is known under these experiments or that the experiments have to be disjoint from $X$, the intervention in the target causal effect. We formalize these notions in the following definition.

\begin{definition}[Surrogate outcome query] \label{def:so_query} A \emph{surrogate outcome query} is a quadruple $(X, Y, G, \mathcal{S})$, where $G = (V, E)$ is a graph, $X, Y \subset V$ are disjoint sets of variables. The set of surrogate outcomes $\mathcal{S} = \{ (Z_1, W_1), \ldots, (Z_n, W_n) \}$ is a collection of intervention--outcome pairs $(Z_i,W_i)$ such that for all $i = 1,\ldots,n$ it holds that $W_i \subset \De(Z_i)_G \setminus Z_i$, $Z_i \subset \An(W_i)_G \setminus W_i$, $\De(W_i)_G \cap Z_i = \emptyset$ and $\An(W_{ij})_G \setminus W_i = \An(W_i)_G \setminus W_i$ for each $W_{ij} \in W_i$.
\end{definition}
While requirements for the sets $Z_i$ and $W_i$ may appear complicated, they are only a formal statement of the fact that we require all variables subject to experimentation to precede all of the outcome variables in the causal order. We also assume that outcomes in a single intervention--outcome pair have the same ancestors. This assumption is made for technical reasons and outcomes with different ancestry can still be represented through separate intervention--outcome pairs. The intuition behind these assumptions is that each intervention--outcome pair should correspond to a single experiment where every manipulated variable has a potential effect on the outcomes. For example, in the graph of Fig.~\ref{fig:intervention_outcome_example}, we would not consider $(\{Z_1,Z_2\},\{W_1,W_2\})$ to be a valid intervention--outcome pair, since manipulating $Z_2$ cannot affect $W_1$.

\begin{figure}[H]
  \centering
  \vspace{0.25cm}
  \begin{tikzpicture}[scale=1.7]
  \node [dot = {0}{0}{Z_1}{below}] at (0,0) {};
  \node [dot = {0}{0}{W_1}{below}] at (1,0) {};
  \node [dot = {0}{0}{Z_2}{below}] at (2,0) {};
  \node [dot = {0}{0}{W_2}{below}] at (3,0) {};

  \draw [->] (Z_1) -- (W_1);
  \draw [->] (W_1) -- (Z_2);
  \draw [->] (Z_2) -- (W_2);
  
  \end{tikzpicture}
  \caption{An example graph on the proper form of intervention--outcome pairs.}
  \label{fig:intervention_outcome_example}
\end{figure}
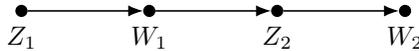
\noindent
Identifiability of a causal effect defined by a surrogate outcome query is characterized by the following definition.

\begin{definition}[Surrogate outcome identifiability] \label{def:so_ident}
Let $(X, Y, G, \mathcal{S})$ be a surrogate outcome query. Let $I_i = \cup_{Z^\prime \subseteq Z_i} \P(w_i \cond \doo(z^\prime), \An(w_i)_{G[\overline{Z^\prime}]} \setminus (w \cup z^\prime))$, where $(Z_i,W_i) \in \mathcal{S}$, and let $\mathcal{I} = \cup_{i = 1}^n I_i \cup \P(v)$. Then the causal effect of $X$ on $Y$ is said to be surrogate outcome identifiable from $\mathcal{I}$ in $G$ if $\P(y \cond \doo(x))$ is uniquely computable from $\mathcal{I}$ in any model that induces $G$.
\end{definition}

The precise formulation of the sets $I_i$ and the experimental distributions is needed to closely connect surrogate outcome identifiability to transportability as we will show later in Section~\ref{sect:sotr}. While the assumption that the interventional distributions are always available for every subset $Z^\prime$ of every $Z_i$ is technical, it can have a real-world interpretation as well. For example, it is realistic to assume that when the joint effect of two medical treatments is studied, either the effect of each individual treatment is already known or they can be estimated from the same experiment. In many cases, it may be unethical to test for the joint effect if it is not known that the individual treatments are safe and efficient.


As an example on surrogate outcome identifiability, we consider the graphs of Fig.~\ref{fig:intro_graph} and attempt to identify the causal effect of $X$ on $Y$ from $\P(x,y,z)$ and $\P(z \cond \doo(x))$. This corresponds to setting $\mathcal{S} = \{(X, Z)\}$ in Definition~\ref{def:so_ident}. It should be noted that this problem cannot be expressed as a $z$-identifiability problem, since the experimental distribution that is available contains an intervention on $X$ and it is not a full experimental distribution over the variables $X, Y$ and $Z$, but is instead restricted to $Z$ only.

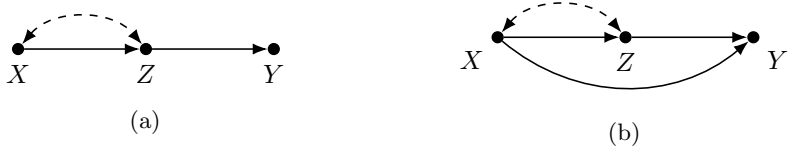
\begin{figure}[ht]
  \begin{subfigure}{0.45\textwidth}
    \centering
    \begin{tikzpicture}[scale=1.7]
    \node [dot = {0}{0}{X}{below}] at (0,0) {};
    \node [dot = {0}{0}{Z}{below}] at (1,0) {};
    \node [dot = {0}{0}{Y}{below}] at (2,0) {};
    
    \draw [->] (X) -- (Z);
    \draw [->] (Z) -- (Y);
    
    \draw [<->,dashed] (X) to [bend left=50]  (Z);
    \end{tikzpicture}
    \caption{}
    \label{fig:intro_graph1}
  \end{subfigure}
  \hfill
  \begin{subfigure}{0.50\textwidth}
    \centering
    \begin{tikzpicture}[scale=1.7]
    \node [dot = {0}{0}{X}{below left}] at (0,0) {};
    \node [dot = {0}{0}{Z}{below}] at (1,0) {};
    \node [dot = {0}{0}{Y}{below right}] at (2,0) {};
    
    \draw [->] (X) -- (Z);
    \draw [->] (Z) -- (Y);
    
    \draw [<->,dashed] (X) to [bend left=50]  (Z);
    \draw [->] (X) to [bend right=40]  (Y);
    \end{tikzpicture}
    \caption{}
    \label{fig:intro_graph2}
  \end{subfigure}
  \caption{Graphs where the causal effect $\P(y \cond \doo(x))$ is not identifiable from $\P(x,y,z)$ alone, but is identifiable via surrogate outcomes using $\P(z \cond \doo(x))$.}
  \label{fig:intro_graph}
\end{figure}
\noindent
We can derive the effect as follows in both Fig.~\ref{fig:intro_graph}(\subref{fig:intro_graph1}) and \ref{fig:intro_graph}(\subref{fig:intro_graph2}):
\begin{align*}
\P(y \cond \doo(x)) &= \sum_{z}\P(y \cond \doo(x), z)\P(z \cond \doo(x)) \\
          &= \sum_{z}\P(y \cond x, z)\P(z \cond \doo(x)).
\end{align*}
Both terms in this expression are computable from $\mathcal{I}$: the term $\P(y \cond x,z)$ can be obtained via conditioning from $\P(x,y,z)$ and the term $\P(z \cond \doo(x))$ is already included in $\mathcal{I}$.
Here the second equality follows from the second rule of do-calculus, since $(Y \indep X \cond Z)_{G[\underline{X}]}$. In this trivial example we can easily determine the correct sequence of applications of do-calculus to reach the desired expression. In general, it is difficult to find such a sequence or determine whether such a sequence even exists. For tasks such as identifiability, the solution was to construct an algorithm that either derives the expression for the effect, or returns a graph structure that can be used to construct two models where the distributions over the observed variables agree, but the interventional distributions differ. Instead of developing a similar algorithm for surrogate outcome identifiability, we will describe this problem as a transportability problem, for which a complete solution already exists in the form of an algorithm \citep{bareinboim2014:completeness}. 

\section{Identifying surrogate outcome queries using transportability} \label{sect:sotr}

In order to describe the connection between surrogate outcomes and transportability we first provide the definition of a \emph{transportability diagram}.

\begin{definition}[Transportability diagram] \label{def:tr_diagram}
Let $( M, M^* )$ be a pair of probabilistic causal models relative to domains $(\pi, \pi^*)$, sharing a graph $G$. The pair $(M, M^*)$ is said to induce a transportability diagram $D$ if $D$ is constructed as follows: every edge in $G$ is also an edge in $D$, $D$ contains an extra edge $T_i \rightarrow V_i$ whenever there might exist a discrepancy $f_{V_i} \neq f^*_{V_i}$ or $\P(u_i) \neq \P^*(u_i)$ between $M$ and $M^*$.
\end{definition}

In the above definition, a domain is simply a formalization of the intuitive notion of different contexts of the same phenomena. The domains serve as indices to differentiate between the different causal models that are depicted by the same graph $G$ and to associate the available observational and experimental distributions with specific models. We illustrate Definition~\ref{def:tr_diagram} via an example. We consider two models, $M$ and $M^*$ that share graph $G$ of Fig.~\ref{fig:tr_example}(\subref{fig:tr_exampleG}) and have the same causal mechanism with the exception that $f_Z \neq f^*_Z$. This discrepancy between the models is now depicted by the transportability diagram of Fig.~\ref{fig:tr_example}(\subref{fig:tr_exampleD}) where the corresponding transportability node and the extra edge have been added. Transportability nodes are denoted by gray squares. We note that transportability diagrams and transportability nodes are sometimes called selection diagrams and selection nodes \citep{bareinboim2013:metatransportability} which should not be confused with the concept of selection bias.

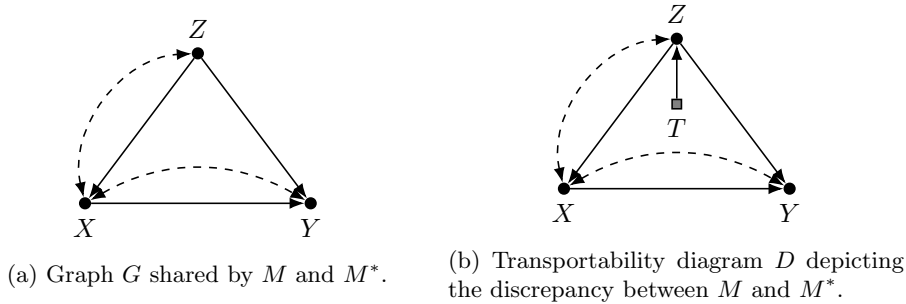
\begin{figure}[ht]
  \begin{subfigure}{0.45\textwidth}
    \centering
    \begin{tikzpicture}[scale=1.5]
    \node [dot = {0}{0}{X}{below}] at (0,0) {};
    \node [dot = {0}{0}{Z}{above}] at (1,1.33) {};
    \node [dot = {0}{0}{Y}{below}] at (2,0) {};
    
    \draw [->] (X) -- (Y);
    \draw [->] (Z) -- (Y);
    \draw [->] (Z) -- (X);
    
    \draw [<->,dashed] (X) to [bend left=50]  (Z);
    \draw [<->,dashed] (X) to [bend left=30]  (Y);
    \end{tikzpicture}
    \caption{Graph $G$ shared by $M$ and $M^*$.}
    \label{fig:tr_exampleG}
  \end{subfigure}
  \hfill
  \begin{subfigure}{0.50\textwidth}
    \centering
    \begin{tikzpicture}[scale=1.5]
    \node [dot = {0}{0}{X}{below}] at (0,0) {};
    \node [dot = {0}{0}{Z}{above}] at (1,1.33) {};
    \node [dot = {0}{0}{Y}{below}] at (2,0) {};
    \node [tr = {0}{0}{T}{below}] at (1,0.75) {};
    
    \draw [->] (X) -- (Y);
    \draw [->] (Z) -- (Y);
    \draw [->] (Z) -- (X);
    \draw [->] (T) -- (Z);
    
    \draw [<->,dashed] (X) to [bend left=50]  (Z);
    \draw [<->,dashed] (X) to [bend left=30]  (Y);
    \end{tikzpicture}
    \caption{Transportability diagram $D$ depicting the discrepancy between $M$ and $M^*$.}
    \label{fig:tr_exampleD}
  \end{subfigure}
  \caption{An example of a transportability diagram where a discrepancy between two domains occurs in the causal mechanism for $Z$.}
  \label{fig:tr_example}
\end{figure}
\noindent

The connection between transportability and surrogate outcome identifiability is not obvious. The general idea is to represent every available experimental distribution as a domain $\pi_i$ where discrepancies described by the transportability nodes $T_i$ take place in variables that have not been measured or randomized in the corresponding experiment, that is in $V \setminus \{W_i \cup  Z_i\}$. In the domain $\pi_i$ experimentation on $W_i$ is available and the goal is to now use the information provided by each domain to derive a \emph{transport formula} for the causal effect. A transportability problem is often implicitly described by the target of identification and available experiments \citep[e.g.][]{bareinboim2013:metatransportability, bareinboim2014:completeness}. Similarly to a surrogate outcome query, we formalize transportability queries in the following definition.

\begin{definition}[Transportability query] \label{def:mz_tr_query} A \emph{transportability query} is an octuple
\[
  (X, Y, \mathcal{D}, D^*, \Pi, \pi^*, \mathcal{Z}, Z^*),
\]
where $\mathcal{D} = \{D^{(1)},\ldots,D^{(n)}\}$ is a collection of transportability diagrams relative to source domains $\Pi = \{\pi_1, \ldots, \pi_n \}$, $D^* = (V, E)$ is the graph of the target domain $\pi^*$, $X, Y \subset V$ are disjoint sets of variables, $\mathcal{Z} = \{ Z_1, \ldots, Z_n \}$ is a collection of sets of variables in which experiments can be conducted in each domain $\pi_i$, and $ Z^*$ is the set of available experiments in the target domain.
\end{definition}

Each transportability diagram in $\mathcal{D}$ depicts the discrepancies between the domains $\pi_i$ and $\pi^*$. Mirroring Definition~\ref{def:so_ident}, transportability of a causal effect defined by a transportability query is characterized by the following definition.

\begin{definition}[Transportability] \label{def:mz_tr}
Let $(X, Y, \mathcal{D}, D^*, \Pi, \pi^*, \mathcal{Z}, Z^*)$ be a transportability query. Let $(\P^{(i)}(v), I^{(i)}_z)$ be the pair of observational and interventional distributions of $\pi_i$, where $I^{(i)}_z = \cup_{Z^\prime \subseteq  Z_i} \P^{(i)}(v \cond \doo(z^\prime))$, and in an analogous manner, let $(\P^*(v), I_z^*)$ be the observational and interventional distributions of $\pi^*$. Let $\mathcal{I} = \cup_{i=1}^n ( \P^{(i)}(v), I^{(i)}_z ) \cup ( \P^*(v), I^*_z )$ be the set of available information. The causal effect $\P^*(y \cond \doo(x))$ is said to be transportable from $\Pi$ to $\pi^*$ in $\mathcal{D}$ with information $\mathcal{I}$ if $\P(y \cond \doo(x))$ is uniquely computable from $\mathcal{I}$ in any model that induces $\mathcal{D}$.
\end{definition}

This definition is referred to as $mz$-transportability in \citep{bareinboim2014:completeness}. Henceforth the superscript $(i)$ is used to refer to the source domain $\pi_i$. A distribution $\P^{(i)}(v)$ governing a source domain is simply a shorthand notation for the conditional distribution where the transportability nodes of the corresponding domain are active, meaning that $\P^{(i)}(v) = \P^*(v \cond t^{(i)})$, where $T^{(i)}$ is the set of all transportability nodes of $\pi_i$.

We present an example on transportability of $\P(y \cond \doo(x))$ using two source domains. The transportability diagrams $D_1$ and $D_2$ associated with the sources are depicted in Fig.~\ref{fig:tr_example2}(\subref{fig:tr_example2D1}) and Fig.~\ref{fig:tr_example2}(\subref{fig:tr_example2D2}) for $\pi_1$ and $\pi_2$, respectively. In transportability diagrams, black squares denote variables for which experimentation is available in the corresponding domain. We assume that experiments on $Z$ are available in $\pi_1$ and on $W$ in domain $\pi_2$. No experiments are available in the target domain $\pi^*$. The graph $G$ of the target domain can be obtained from either $D_1$ or $D_2$ by simply omitting the transportability nodes. The corresponding transportability query for this problem is
\[
Q = (X,Y,\{D_1,D_2\},G,\{\pi_1,\pi_2\},\pi^*,\{\{Z\},\{W\}\},\emptyset).
\]

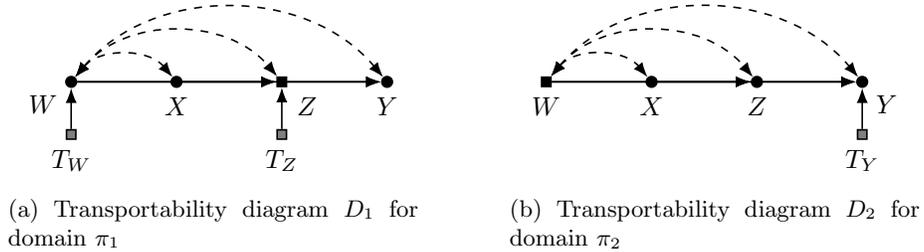
\begin{figure}[ht]
  \begin{subfigure}{0.45\textwidth}
    \centering
    \begin{tikzpicture}[scale=1.4]
    \node [dot = {0}{0}{W}{below left}] at (-1,0) {};
    \node [dot = {0}{0}{X}{below}] at (0,0) {};
    \node [intv = {0}{0}{Z}{below right}] at (1,0) {};
    \node [dot = {0}{0}{Y}{below}] at (2,0) {};
    \node [tr = {0}{0}{T_W}{below}] at (-1,-0.5) {};
    \node [tr = {0}{0}{T_Z}{below}] at (1,-0.5) {};
    
    \draw [->] (W) -- (Z);
    \draw [->] (X) -- (Z);
    \draw [->] (Z) -- (Y);
    \draw [->] (W) -- (Y);
    \draw [->] (T_W) -- (W);
    \draw [->] (T_Z) -- (Z);
    
    \draw [<->,dashed] (W) to [bend left=50] (X);
    \draw [<->,dashed] (W) to [bend left=50] (Y);
    \draw [<->,dashed] (W) to [bend left=50] (Z);
    \end{tikzpicture}
    \caption{Transportability diagram $D_1$ for domain $\pi_1$}
    \label{fig:tr_example2D1}
  \end{subfigure}
  \hfill
  \begin{subfigure}{0.45\textwidth}
    \centering
    \begin{tikzpicture}[scale=1.4]
    \node [intv = {0}{0}{W}{below}] at (-1,0) {};
    \node [dot = {0}{0}{X}{below}] at (0,0) {};
    \node [dot = {0}{0}{Z}{below}] at (1,0) {};
    \node [dot = {0}{0}{Y}{below right}] at (2,0) {};
    \node [tr = {0}{0}{T_Y}{below}] at (2,-0.5) {};
    
    \draw [->] (W) -- (Z);
    \draw [->] (X) -- (Z);
    \draw [->] (Z) -- (Y);
    \draw [->] (W) -- (Y);
    \draw [->] (T_Y) -- (Y);
    
    \draw [<->,dashed] (W) to [bend left=50] (X);
    \draw [<->,dashed] (W) to [bend left=50] (Y);
    \draw [<->,dashed] (W) to [bend left=50] (Z);
    \end{tikzpicture}
    \caption{Transportability diagram $D_2$ for domain $\pi_2$}
    \label{fig:tr_example2D2}
  \end{subfigure}
  \caption{Transportability diagrams related to two source domains}
  \label{fig:tr_example2}
\end{figure}
\noindent
The transport formula can be derived using do-calculus as follows:
\begin{align*}
\P^*(y \cond \doo(x)) &= \sum_{z} \P^*(z \cond \doo(x))\P^*(y \cond \doo(x), z) \\
                     &= \sum_{z} \P^*(z \cond \doo(x,w))\P^*(y \cond \doo(x,z)) \\
                     &= \sum_{z} \P^*(z \cond x, \doo(w))\P^*(y \cond \doo(z)) \\
                     &= \sum_{z} \P^*(z \cond x, \doo(w), t_y)\P^*(y \cond \doo(z), t_z, t_w) \\
                     &= \sum_{z} \P^{(2)}(z \cond x, \doo(w))\P^{(1)}(y \cond \doo(z))
\end{align*}
Where the equalities follow from the following sequence: second equality from rules three and two by $(Z \indep W \cond X)_{G_{[\overline{X},\overline{W}]}}$ and $(Y \indep Z \cond X)_{G[\overline{X},\underline{Z}]}$, third equality from rules two and three by $(Z \indep X \cond W)_{G[\overline{W},\underline{X}]}$ and $(Y \indep X \cond Z)_{G[\overline{Z},\overline{X}]}$, fourth equality from rule one by $(Z \indep T_Y \cond X,W)_{G[\overline{W}]}$ and $(Y \indep \{T_Z,T_W\} \cond Z)_{G[\overline{Z}]}$. The last equality is just a rewrite of the terms in the shorthand notation for active transportability nodes of a specific domain.

Next, we will outline the procedure to transform a surrogate outcome identifiability query into a transportability query.

\begin{definition}[Query transformation] \label{def:so_into_tr} Let $Q_S = (X, Y, G, \mathcal{S} )$ be a surrogate outcome query that is to be transformed into a transportability query $Q_T = (X, Y, \mathcal{D}, G, \Pi, \pi^*, \mathcal{Z}, \emptyset)$, where sets $X$ and $Y$ remain unchanged. The graph of the target domain $\pi^*$ is $G$. The set of source domains $\Pi = \{\pi_1, \ldots, \pi_n\}$ and the collection of their respective transportability diagrams $\mathcal{D} = \{D^{(1)},\ldots,D^{(n)}\}$  are constructed from $G$ as follows: $D^{(i)}$ contains an edge $T^{(i)}_j \rightarrow V_j$ for every vertex $V_j \in (\De(Z_i)_G \setminus W_i) \cup (C_{W_i} \setminus An(W_i)_{G[\overline{Z_i}]})$, where $C_{W_i} = \bigcup_j C_{W_{ij}}$ and $C_{W_{ij}}$ is the set of vertices of the c-component that contains the vertex $W_{ij} \in W_i$. The collection of available experiments is obtained directly from $\mathcal{S}$ by setting $\mathcal{Z} = \{Z_1, \ldots, Z_n\} $ ($ Z^* = \emptyset$ for $\pi^*$).
\end{definition}

The transformation provided by Definition~\ref{def:so_into_tr} serves as our basis for solving a given surrogate outcome identifiability problem. Transportability nodes are used to denote our lack of experimental information and to exert control over which transformed transportability queries should be identifiable. For each set $Z_i$, we know that the flow of information caused by the intervention of $Z_i$ will not propagate to non-descendants of $Z_i$, which is why we add a transportability node for each vertex in $\De(Z_i)_G \setminus W_i$. However, confounding must also be taken into account in the outcome set $W_i$, which is why a transportability node is added for each vertex of each c-component that shares a vertex with $W_i$ with the exception of ancestors of $W_i$. Later we will show that a causal effect is surrogate outcome identifiable if the corresponding causal effect obtained from the query transformation is transportable.

We return to the example on surrogate outcome identifiability in Section~\ref{sect:defi} and show how the surrogate outcome query is transformed into a transportability query in this instance. The task is to identify $\P(y \cond \doo(x))$ from $\P(x,y,z)$ and $\P(z \cond \doo(x))$ in the graph $G$ of Fig.~\ref{fig:intro_graph}(\subref{fig:intro_graph1}). The corresponding surrogate outcome query is 
\[
  Q_S = (X, Y, G, \{(X,Z)\}).
\]
The set $\mathcal{S}$ consists of a single element $(X,Y)$, which means that our transformed query will have a single source domain $\pi_1$. The transportability diagram for this domain is constructed according to Definition~\ref{def:so_into_tr} by adding a transportability node for each vertex in $\De(X)_G \setminus \{Z\} = \{X,Y\}$. The set $C_Z \setminus An(Z)_G$ is empty so no other transportability nodes have to be added. The resulting transportability diagram $D^{(1)}$ is shown in Fig.~\ref{fig:sotr_example}. The transformed query is now
\[
  Q_T = (X,Y,\{D^{(1)}\},G,\{\pi_1\},\pi^*,\{X\},\emptyset).
\]

\begin{figure}[H]
  \centering
  \begin{tikzpicture}[scale=1.7]
  \node [intv = {0}{0}{X}{below}] at (0,0) {};
  \node [dot = {0}{0}{Z}{below}] at (1,0) {};
  \node [dot = {0}{0}{Y}{below}] at (2,0) {};
  \node [tr = {0}{0}{T_X}{above}] at (0,0.5) {};
  \node [tr = {0}{0}{T_Y}{above}] at (2,0.5) {};
  
  \draw [->] (X) -- (Z);
  \draw [->] (Z) -- (Y);
  \draw [->] (T_X) -- (X);
  \draw [->] (T_Y) -- (Y);
  
  \draw [<->,dashed] (X) to [bend left=50]  (Z);
  \end{tikzpicture}
  \caption{A transportability diagram resulting from a query transformation.}
  \label{fig:sotr_example}
\end{figure}
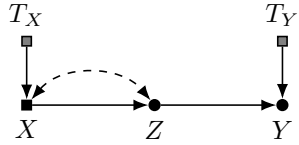
\noindent

Next, we present an algorithm labeled TRSO for computing transportability formulas that is a modification of the algorithm presented in \citep{bareinboim2014:completeness}. The purpose of this modified algorithm is to solve transportability queries that have been obtained through a query transformation of a surrogate outcome problem. In the original formulation, experimental information from the source domains is used only if identification in the target domain fails. Instead, we will prioritize experiments over observations to make full use of the available information.

\begin{algorithm}[H]
  \begin{algorithmic}[1]
  \INPUT{value assignments $x,y$, local distribution $P$ relative to domain index $S$ ($\pi_0$ denotes $\pi$), active experiments $I$, local transportability diagram $D$ of domain $S$, set of available experiments $\mathcal{Z}$. The set $T_i$ denotes the transportability nodes in $\pi_i$ ($T_0 = \emptyset$ for $\pi^*$). The set of all transportability diagrams $\mathcal{D}$ and a topological ordering $\varphi$ of $D$ are globally available.}
  \OUTPUT{$\P^*(y \cond \doo(x))$ in terms of transportability information $\mathcal{I}$, or FAIL.}
  \Statex
  \Statex {\bf function} \text{TRSO}$(y,x,P,I,S,D,\mathcal{Z})$
  \State {\bf if} $x = \emptyset$ {\bf return} $\sum_{v\setminus y} P.$
  \State {\bf if} $V \setminus \An(Y)_D \neq \emptyset$,
  \Statex \quad {\bf return} $\text{TRSO}(y,x \cap \An(y)_D, \sum_{V \setminus \An(Y)_D} P, S, D[\An(Y)_D],\mathcal{Z})$.
  \State {\bf let} $W = (V \setminus X) \setminus \An(Y)_{D[\overline{X}]}$.
  \Statex \quad {\bf if} $W \neq \emptyset $, {\bf return} $\text{TRSO}(y,x \cup w, P, I, S, D,\mathcal{Z})$.
  \State {\bf if} $C(D[V \setminus X]) = \{D[C_1],\ldots,D[C_n]\}$,
  \Statex \quad {\bf return} $\sum_{V \setminus (X \cup Y)} \prod_{i = 1}^n \text{TRSO}(c_i, v \setminus c_i, P, I, S, D,\mathcal{Z})$.
  \State {\bf if} $C(D[V \setminus X]) = \{D[C]\}$,
  \State \quad {\bf if} $I = \emptyset$, for $i = 0,\ldots,|\mathcal{D}|$,
  \Statex \quad \quad {\bf if} $(T_i \indep Y \cond X)_{D^{(i)}[\overline{X}]}$ {\bf and} $Z_i \cap X \neq \emptyset$, 
  \Statex \quad \quad {\bf let} $E_i = \text{TRSO}(y, x \setminus z_i, P, Z_i \cap X, i, D[V \setminus (Z_i \cap X)],\mathcal{Z})$.
  \State \quad {\bf if} $E_k \neq$ {\bf FAIL for some} $k$, {\bf return} $E_k$.
  \State \quad {\bf if} $C(D) \neq \{D\}$,
  \State \quad \quad {\bf if} $D[C] \in C(D)$, {\bf return} $\sum_{C \setminus Y} \prod_{V_i \in C} (\sum_{V \setminus V_\varphi^{(i)}} P) / (\sum_{V \setminus V_\varphi^{(i-1)}} P)$. 
  \State \quad \quad {\bf let} $C^\prime \subset V$ {\bf such that} $D[C^\prime] \in C(D)$ {\bf and} $C \subset C^\prime$.
  \Statex \quad \quad {\bf if} $I = \emptyset$, let $\mathcal{Z}^\prime = \emptyset$,
  \Statex \quad \quad {\bf else},
  \Statex \quad \quad \quad {\bf if} $\Pa(C^\prime)_D \cap T_S = \emptyset$, {\bf let} $\mathcal{Z}^\prime = \mathcal{Z}$,
  \Statex \quad \quad \quad {\bf else return FAIL}.
  \Statex \quad \quad {\bf return} 
  \Statex \quad \quad \quad $\text{TRSO}(y,x \cap c^\prime, \prod_{V_i \in C^\prime} P(v_i \cond V_\varphi^{(i - 1)} \cap C^\prime, v_\varphi^{(i-1)} \setminus c^\prime), I, S, D[C^\prime], \mathcal{Z}^\prime)$.
  \State \quad {\bf else return FAIL}.
  \end{algorithmic}
  \caption{A modified transportability algorithm for query transformations.}
  \label{alg:tr} 
\end{algorithm}

Some restrictions have to be imposed, since when transportability of causal effects is considered we always have access to the full experimental distributions $\P^{(i)}(v \cond \doo(z_i))$ in any domain $\pi_i$. This has to be taken into account by preventing certain operations on the joint distributions to be carried out when query transformations for surrogate outcomes are considered. For example when line 10 is triggered, we check whether the local c-component is affected by transportability nodes and prevent the use of experimental information if this is the case. The original formulation of the algorithm also includes a weighting scheme for effects that can be identified from multiple domains. We omit this part for clarity and use the first domain where an identifiable effect was encountered. The following theorem formally describes in the purpose of TRSO.

\begin{theorem} \label{thm:sufficient} Let $(X, Y, \mathcal{D}, G, \Pi, \pi^*, \mathcal{Z}, \emptyset )$ be the query transformation of a surrogate outcome query $(X, Y, G, \mathcal{S} )$. Then $\P(y \cond \doo(x))$ is surrogate outcome identifiable from $\mathcal{I}^1$ in $G$ if $\mathrm{TRSO}(y,x,\P^*(v),\emptyset,0,G,\mathcal{D},\varphi)$ succeeds using $\mathcal{I}^2$ and $\mathcal{Z}$, where information set $\mathcal{I}^1$ is the surrogate outcome information in Definition~\ref{def:so_ident} and $\mathcal{I}^2$ is the transportability information in Definition~\ref{def:mz_tr}.
\end{theorem}

Technical details and auxiliary results required to prove Theorem~\ref{thm:sufficient} are presented in the next section.

We recall the example from Section~\ref{sect:intro} on identifying $\P(y_1,y_2 \cond \doo(x_1, x_2))$ from $\P(v), \P(y_2 \cond \doo(x_2), x_1, z,w)$ and $\P(y_1 \cond \doo(x_1),z,w)$ in the graph of Fig.~\ref{fig:intro_example}, and solve its query transformation using TRSO. The set $\mathcal{S}$ of surrogate outcomes contains two intervention--outcome pairs, $S_1 = (X_1, Y_1)$ and $S_2 = (X_2, Y_2)$. For $S_1$, transportability nodes are added for 
\[
  (\De(X_1)_G \setminus \{Y_1\}) \cup (C_{Y_1} \setminus \An(Y_1)_{G[\overline{X_1}]}) = (\{X_1,Y_2,Y_1\} \setminus \{Y_1\}) \cup \emptyset = \{X_1,Y_2\}.
\]
For $S_2$, transportability nodes are added for 
\[
  (\De(X_2)_G \setminus \{Y_2\}) \cup (C_{Y_2} \setminus \An(Y_2)_{G[\overline{X_2}]}) = 
  (\{X_2,Y_2\} \setminus \{Y_2\} \cup \emptyset = \{X_2\}.
\]
The corresponding transportability diagrams $D^{(1)}$ and $D^{(2)}$ for the domains $\pi_1$ and $\pi_2$ of the query transformation are shown in Fig.~\ref{fig:intro_example_tr}.
\begin{figure}[H]
  \begin{subfigure}[t]{0.49\textwidth}
  \centering
  \begin{tikzpicture}[scale=1.5]
  \node [intv = {0}{0}{X_1}{above}] at (1.5,2) {};
  \node [dot = {0}{0}{X_2}{below right}] at (3,0) {};
  \node [dot = {0}{0}{Z}{below left}] at (0,0) {};
  \node [dot = {0}{0}{Y_1}{above left}] at (0.6,0.75) {};
  \node [dot = {0.1}{0.1}{Y_2}{above}] at (2.4,0.75) {};
  \node [dot = {0}{0}{W}{above}] at (1.5,1.25) {};

  \node [tr = {0}{0}{T_{Y_2}^{(1)}}{above}] at (3,1) {};
  \node [tr = {0}{0}{T_{X_1}^{(1)}}{above}] at (2.1,2.25) {};
  
  \draw [->] (Z) -- (X_2);
  \draw [->] (Z) -- (Y_2);
  \draw [->] (Z) -- (Y_1);
  \draw [->] (W) -- (Y_1);
  \draw [->] (W) -- (Y_2);
  \draw [->] (X_1) -- (Y_2);
  \draw [->] (X_1) -- (Y_1);
  \draw [->] (X_2) -- (Y_2);

  \draw [->] (T_{X_1}^{(1)}) -- (X_1);
  \draw [->] (T_{Y_2}^{(1)}) -- (Y_2);

  \draw [<->,dashed] (X_1) to [bend right=50]  (Y_1);
  \draw [<->,dashed] (Z) to [bend right=10]  (W);
  \draw [<->,dashed] (Z) to [bend right=30]  (X_2);
  \end{tikzpicture}
  \caption{$D^{(1)}$}
  \label{fig:intro_example_trdom1}
  \end{subfigure}
  \begin{subfigure}[t]{0.49\textwidth}
  \centering
  \begin{tikzpicture}[scale=1.5]
  \node [dot = {0}{0}{X_1}{above}] at (1.5,2) {};
  \node [intv = {0}{0}{X_2}{below right}] at (3,0) {};
  \node [dot = {0}{0}{Z}{below left}] at (0,0) {};
  \node [dot = {0}{0}{Y_1}{above left}] at (0.6,0.75) {};
  \node [dot = {0.1}{0.1}{Y_2}{above}] at (2.4,0.75) {};
  \node [dot = {0}{0}{W}{above}] at (1.5,1.25) {};
    \node [tr = {0}{0}{T_{X_2}^{(2)}}{above}] at (3,0.75) {};
  
  \draw [->] (Z) -- (X_2);
  \draw [->] (Z) -- (Y_2);
  \draw [->] (Z) -- (Y_1);
  \draw [->] (W) -- (Y_1);
  \draw [->] (W) -- (Y_2);
  \draw [->] (X_1) -- (Y_2);
  \draw [->] (X_1) -- (Y_1);
  \draw [->] (X_2) -- (Y_2);

  \draw [->] (T_{X_2}^{(2)}) -- (X_2);

  \draw [<->,dashed] (X_1) to [bend right=50]  (Y_1);
  \draw [<->,dashed] (Z) to [bend right=10]  (W);
  \draw [<->,dashed] (Z) to [bend right=30]  (X_2);
  \end{tikzpicture}
  \caption{$D^{(2)}$}
  \label{fig:intro_example_trdom2}
  \end{subfigure}
  \caption{Transportability diagrams obtained through the query transformation from Fig.~\ref{fig:intro_example} of the introductory example.}
  \label{fig:intro_example_tr}
\end{figure}
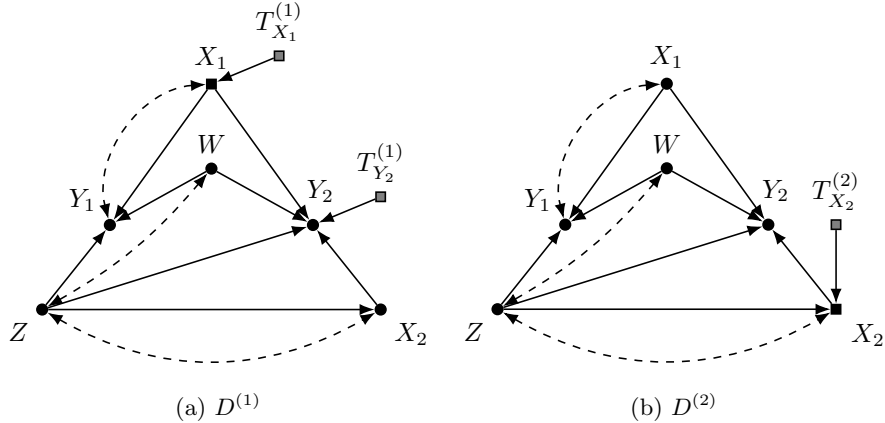
\noindent
By tracing the algorithm, we trigger line 4 first and obtain
\begin{align*}
  \P^*(y_1,y_2 \cond \doo(x_1, x_2)) &= \sum_{w,z} \left( \vphantom{\frac12} \P^*(y_2 \cond \doo(x_1,x_2,z,w,y_1))  \right.\,  \times \\
  & \quad \left. \P^*(y_1 \cond \doo(x_1, x_2,z,w,y_2)) \P^*(z,w \cond \doo(x_1, x_2,y_1,y_2)) \vphantom{\frac12} \right).
\end{align*}
Since $\An(Z,W)_G = \{Z,W\}$, line 2 and then line 1 are triggered for the last term, which is simply $\P^*(z,w)$. The recursive calls for the first two terms both trigger line 2 due to $Y_1$ not being an ancestor of $Y_2$ and $\{X_2,Y_2\}$ not being ancestors of $Y_1$. After these calls we have
\[
\P^*(y_1,y_2 \cond \doo(x_1, x_2)) = \!\sum_{w,z} \P^*(y_2 \cond \doo(x_1,x_2,z,w)) \P^*(y_1 \cond \doo(x_1,z,w)) \P^*(z,w).
\]
Line 10 is triggered next for both of the first two terms because 
\[(Y_1 \indep \{T^{(1)}_{X_1},T^{(1)}_{Y_2}\} \cond \{X_1,Z,W,X_2\})_{D^{(1)}[\overline{X_1,Z,W,X_2}]}.
\] 
and 
\[
(Y_2 \indep T^{(2)}_{X_2} \cond \{X_1,Z,W\})_{D^{(2)}[\overline{X_1,Z,W}]}.
\]
This means that intervention on $X_1$ is activated for the first term and intervention on $X_2$ is activated for the second terms. Finally, line 7 is triggered for both remaining terms and we obtain
\begin{align*}
\P^*(y_1,y_2|\doo(x_1, x_2)) &= \\
\sum_{w,z}& \P^{(1)}(y_2 \cond \doo(x_2), x_1, z, w)) \P^{(2)}(y_1 \cond \doo(x_1), z,w) \P^*(z,w),
\end{align*}
as the final expression. We obtain a solution for the original surrogate outcome problem by simply omitting the domain indicators from this expression
\[
\P(y_1,y_2|\doo(x_1, x_2)) = \sum_{w,z} \P(y_2 \cond \doo(x_2), x_1, z, w)) \P(y_1 \cond \doo(x_1), z,w) \P^*(z,w).
\]
We can also derive the effect using the \emph{causaleffect} R package with the following commands:
\begin{Verbatim}[fontsize = \small]
library(causaleffect)
library(igraph)

> fig1 <- graph.formula(x_1 -+ y_2, x_1 -+ y_1, w -+ y_1, w -+ y_2,
+  z -+ y_1, x_2 -+ y_2, z -+ y_2, z -+ x_2, w -+ z, z -+ w,
+  z -+ x_2, x_2 -+ z, y_1 -+ x_1, x_1 -+ y_1, simplify = FALSE)
> fig1 <- set.edge.attribute(fig1, "description", 9:14, "U")

> s1 <- list(
+  list(Z = c("x_2"), W = c("y_2")),
+  list(Z = c("x_1"), W = c("y_1"))
> )

> cat(surrogate.outcome(y = c("y_1", "y_2"), x = c("x_1","x_2"), 
+  S = s1, G = fig1))
\sum_{w,z}P_{x_2}(y_2|x_1,w,z)P(w,z)P_{x_1}(y_1|w,z)
\end{Verbatim}
The package uses the notation $\P_{x_1}(y_1 \cond w,z)$ to denote $\P(y_1 \cond \doo(x_1),w,z)$. 

In the next section we prove the correctness of TRSO and show that the omission of the domain indicators from its output always produces a valid expression for the original surrogate outcome identifiable causal effect.

\section{Correctness of the modified transportability algorithm} \label{sect:techical}

First, we recall that do-calculus is complete with respect to transportability and prove some useful lemmas.

\begin{theorem}[do-calculus characterization] \label{thm:do_calculus_tr} The rules of do-calculus together with standard probability manipulations are complete for establishing transportability of causal effects.
\end{theorem}
\begin{proof}
See \citep{bareinboim2014:completeness}.
\end{proof}

Theorem~\ref{thm:do_calculus_tr} shows that a sequence of valid operations necessarily exists for a transportable causal effect. We define this sequence explicitly.

\begin{definition}[do-calculus sequence] \label{def:do_seq} Let $G$ be a graph or a transportability diagram, let $p$ be an identifiable or transportable causal effect and let $\mathcal{I}$ be a set of available information. A \emph{do-calculus sequence} for $p$ in $G$ is a pair
\[ \delta_p = ( \mathcal{R}, \mathcal{P} ), \]
where $\mathcal{R}$ is an $n-$tuple $(R_1,\ldots,R_n)$ such that each $R_i$ is either a member of the index set $\{m, c, r\}$ or a quintuple $(Y_i, Z_i, X_i, W_i, r_i)$ such that $(Y_i \indep Z_i \cond X_i, W_i)_{G^\prime}$ and $G^\prime = G[\overline X_i]$ if $r_i = 1$, $G^\prime = G[\overline X_i, \underline Z_i]$ if $r_i = 2$ and $G^\prime = G[\overline X, \overline Z_i(W_i)]$ if $r_i = 3$ and $Z_i(W_i) = Z_i \setminus \An(W_i)_{G[\overline X]}$. $\mathcal{P} = (p_1,\ldots,p_n)$ is a sequence of probability distributions such that if $R_i$ is of the first type described above, $p_{i}$ is obtained from $p_{i-1}$ via marginalization for $R_i = m$, conditioning for $R_i = c$ and the chain-rule if $R_i = r$. If $R_i$ is of the second type, then $p_{i}$ is obtained from $p_{i-1}$ using rule number $r_i$ of do-calculus licensed by the sets $X_i, Y_i, Z_i$ and $W_i$. Furthermore, when $p_0 = p$ is transformed as dictated by the sequence $\delta_p$, an expression $p_n$ is obtained such that each term that appears in $p_n$ is a member of $\mathcal{I}$ or computable from $\mathcal{I}$ without do-calculus.
\end{definition}

The idea is to use a do-calculus sequence of a transportable causal effect to construct a do-calculus sequence for its query transformation counterpart. However, as do-calculus statements stem from d-separation in the underlying graph, we must first establish that d-separation is invariant to the presence of transportability nodes.

\begin{lemma} \label{lem:d_sep_tr} Let $D$ be a transportability diagram and let $G$ be its subgraph obtained by removing all transportability nodes $T$ from $D$. Let $X,Y,Z$ be disjoint sets of vertices of $D$ such that they do not contain transportability nodes. Then $X$ and $Y$ are d-separated by $Z \cup T^\prime$ for every $T^\prime \subseteq T$ in $D$ if and only if $X$ and $Y$ are d-separated by $Z$ in $G$.
\end{lemma}
\begin{proof}
(i) Suppose that $X$ and $Y$ are d-separated by $Z\,\cup\,T^\prime$ in $D$. By assumption $X, Y$ and $Z$ do not contain any transportability nodes. This means that no path from $X$ to $Y$ contains transportability nodes, since a path containing such a node would necessarily have it as one of the path's endpoints by Definition~\ref{def:tr_diagram}. Furthermore, a transportability node cannot be a descendant of a collider by definition. Thus all paths from $X$ to $Y$ remain d-separated if we remove all transportability nodes from $D$.

(ii) Suppose that $X$ and $Y$ are d-separated by $Z$ in $G$. Adding transportability nodes to $G$ cannot create any new paths between $X$ and $Y$ since a transportability node is only connected to other vertices of the graph through a single vertex. As before, a transportability node cannot be a descendant of a collider by definition. Thus all paths between $X$ and $Y$ are d-separated by $Z \cup T^\prime$ in $D$ for any subset $T^\prime \subseteq T$.
\end{proof}

\begin{corollary} \label{cor:cond_indep} Let $Q_S$ be a surrogate outcome query with a graph $G$ and let $Q_T$ be its query transformation with a collection of transportability diagrams $\mathcal{D}$. Then any conditional independence statement $(Y \indep X \cond Z)_{D_i}$ that holds in some transportability diagram $D_i$ of $\mathcal{D}$ also holds in $G$ if the sets $Y$ and $X$ do not contain transportability nodes. Conversely, every conditional independence statement of $G$ holds in every diagram of $\mathcal{D}$.
\end{corollary}
\begin{proof}
The proof immediately follows from Theorem~\ref{thm:do_calculus_tr} by noting that $G$ can be obtained from each element of $\mathcal{D}$ by removing all transportability nodes.
\end{proof}

We show that there always exists do-calculus sequence such that every operation that manipulates transportability nodes does not manipulate any other vertices at the same time.
\begin{lemma} \label{lem:do_sequence_exists}
Let $(X,Y,\mathcal{D},D^*,\Pi,\pi^*,\mathcal{Z},Z^*)$ be a transportability query and let $T$ be the set of all transportability nodes over the domains of $\mathcal{D}$ and the target diagram $D^*$. If $p = \P(y \cond \doo(x))$ is a transportable causal effect with transportability information $\mathcal{I}$ of Definition~\ref{def:mz_tr}, then there exists a do-calculus sequence $d_p = (\mathcal{R},\mathcal{P})$ such that whenever $R_i$ is of the form $(Y_i,Z_i,X_i,W_i,r_i)$ then either $Z_i \cap T = \emptyset$ or $Z_i \subset T$.
\end{lemma}
\begin{proof}
Let $d_p^\prime = (\mathcal{R}^\prime,\mathcal{P}^\prime)$ be any do-calculus sequence for $p$. It suffices to consider $R^\prime_i \in \mathcal{R}^\prime$ of the form $(Y_i,Z_i,X_i,W_i,r_i)$. If $r_i \in \{2,3\}$ there is nothing to prove, since the second and third rules of do-calculus manipulate interventions which are not allowed on transportability nodes. Let $r_i = 1$ and suppose that $Z_i \setminus T \neq \emptyset$. Then from Definition~\ref{def:do_seq} we know that $(Y_i \indep Z_i \cond X_i, W_i)_{G[\overline{X}]}$ which implies that $(Y_i \indep Z_i \setminus T \cond X_i, W_i)_{G[\overline{X}]}$ and $(Y_i \indep Z_i \cap T \cond X_i, W_i)_{G[\overline{X}]}$. Now, let $\mathcal{R}$ contain every member of $\mathcal{R}^\prime$ except that each $R_i^\prime$ with $r_i = 1$ and $Z_i \setminus T \neq \emptyset$ is replaced by $R_i^{1} = (Y_i,Z_i \cap T,X_i,W_i,r_i)$ and $R_i^{2} = (Y_i,Z_i \setminus T,X_i,W_i,r_i)$. Similarly, let $\mathcal{P}$ contain every member of $\mathcal{P}^\prime$ except that each $p_i^\prime$, where the corresponding $R_i^\prime$ has the aforementioned property, is replaced by $p_i^{1}$ and $p_i^{2}$ where $p_i^{1}$ is obtained from $p_{i-1}^\prime$ by applying the first rule of do-calculus with the sets $Y_i, Z_i \cap T, X_i$ and $W_i$, and $p_i^{2}$ is obtained from $p_i^{1}$ by applying the first rule of do-calculus with the sets $Y_i, Z_i \setminus T, X_i$ and $W_i$. By construction, $d_p = (\mathcal{R}, \mathcal{P})$ is a do-calculus sequence for $p$ with the desired property.
\end{proof}

\begin{theorem} \label{thm:soundness} TRSO is sound.
\end{theorem}
\begin{proof} Lines 1 through 9 are identical to the original formulation of the transportability algorithm and their soundness was established in \citep{bareinboim2014:completeness} with the exception that the order of lines 6 and 10 is reversed. Line 10 is different from the original. On this line we first find the c-component $D[C^\prime]$ of $D$ such that $C \subset C^\prime$. This c-component necessarily exists since the c-components of $D[V \setminus X]$ are always subsets of the c-components of $D$. The c-component $D[C^\prime]$ is unique because the vertex sets of c-components of any graph are disjoint. Next we check if there is an active intervention. If there is no such intervention ($I = \emptyset$), we remove the ability for experimentation entirely by setting $\mathcal{Z}^\prime = \emptyset$. If there is an active experiment ($I \neq \emptyset$), we check whether it falls into the category of allowed experiments by evaluating if $\Pa(C^\prime)_D \cap T_S = \emptyset$. If it does not, the recursive call fails. If there were no active experiments ($I = \emptyset$) or active experimentation is permissible ($\Pa(C^\prime)_D \cap T_S = \emptyset$), we simply continue the recursion in the c-component $D[C^\prime]$. The checks for allowing experimentation do not affect the soundness of the recursive function call that follows them on line 10. This recursive call was shown to be sound in \citep{bareinboim2014:completeness}.
\end{proof}
The next result characterizes an important feature of the transport formulas produced by a successful application of TRSO.

\begin{lemma} \label{lem:dist_property} Let $(X, Y, \mathcal{D}, G, \Pi, \pi^*, \mathcal{Z}, \emptyset )$ be the query transformation of a surrogate outcome query $(X, Y, G, \mathcal{S} )$. If $\mathrm{TRSO}$ succeeds in transporting the causal effect $p = \P^*(y \cond \doo(x))$, then for every term that appears in the expression for $p$ of the form $\P^{(i)}(c \cond \doo(z^\prime), d^\prime)$ one of the following holds: either
\begin{align}
  \P^{(i)}(c \cond \doo(z^\prime), d^\prime) &= \P^{(i)}(w^* \cond \doo(z^\prime), x^\prime) \P^*(x^\prime), \\
\intertext{or}
  \P^{(i)}(c \cond \doo(z^\prime), d^\prime) &= \P^*(c \cond \An(c)_{G[\overline{Z^\prime}]} \setminus c), \\
\intertext{or}
  \P^{(i)}(c \cond \doo(z^\prime), d^\prime) &= \P^{(i)}(w^* \cond \doo(z^\prime), x^\prime),
\end{align}
where $x^\prime = \An(w^*)_{G[\overline{Z^\prime}]} \setminus (w^* \cup z^\prime)$, $W^* \subset W_i$ and there exists a set $Z_i$ such that $Z^\prime \subseteq Z_i$ and $(Z_i,W_i) \in \mathcal{S}$. Furthermore, the right-hand sides of (1), (2) and (3) are identifiable from the information set $\mathcal{I}$ of Definition~\ref{def:so_ident} when the domain indicators are omitted.
\end{lemma}
We defer the proof of Lemma~\ref{lem:dist_property} to Appendix A. We are ready to prove Theorem~\ref{thm:sufficient}

\begin{proof}[Proof of Theorem~\ref{thm:sufficient}] Assume that $p = \P^*(y \cond \doo(x))$ is transportable from $\Pi$ to $\pi^*$ in $\mathcal{D}$ with information $\mathcal{I}^2$ by applying $\text{TRSO}(y,x,\P^*(v),\emptyset,0,G,\mathcal{D},\varphi)$. Let $\delta_d = (\mathcal{R}, \mathcal{P})$ be a do-calculus sequence for $p$ of the form given by Lemma~\ref{lem:do_sequence_exists}. This sequence is valid since the algorithm is sound by Theorem~\ref{thm:soundness}. We can categorize the do-calculus steps into two distinct types: those that do not modify the transportability nodes present in the expression, and those that do. In other words, if $T$ is the set of all transportability nodes over the domains $\mathcal{D}$ and $D^*$, the first category contains steps $R_i = (Y_i,Z_i,X_i,W_i,r_i)$ such that $Z_i \cap T = \emptyset$. By Corollary~\ref{cor:cond_indep}, the conditional independence statements in a transportability diagram involving sets $Z_i$ of this type are also valid in a corresponding graph where transportability nodes have been removed. This means that if $(Y_i \indep Z_i \cond X, W)_{D_j}$ then $(Y_i \indep Z_i \cond X, W \setminus T)_{D_j[V \setminus T]}$. This allows us to construct a new do-calculus sequence as follows: For any $R_i \in \mathcal{R}$ of the form $(Y_i,Z_i,X_i,W_i,r_i)$ with $Z_i \cap T = \emptyset$ we let $R^\prime_i = (X_i,Y_i,Z_i,W_i \setminus S,r_i)$. If $Z_i \cap T \subseteq T$, we let $R_i^\prime = \emptyset$. For any $R_i$ in the index set ${m,c,r}$ we let $R_i^\prime = R_i$. Let the sequence $\mathcal{R^\prime}$ now consists of those $R_i^\prime$ that are non-empty. The sequence $\mathcal{P}^\prime$ of distributions is constructed from $q = \P(y \cond \doo(x))$ through the sequence of manipulations described by $\mathcal{R}^\prime$. We apply Lemma~\ref{lem:dist_property} for each term of the form $\P(c \cond \doo(z), d^\prime)$ in the resulting formula for $q$ such that there exists no pair $(Z_i,W_i) \in \mathcal{S}$ with $W_i = C$. This means that additional manipulation steps $R_i^\prime$ and distributions $p^\prime_i$ are added that correspond to the transformation of the distribution on the left-hand side to the distribution on the right-hand side in one of the conditions of Lemma~\ref{lem:dist_property}.

It remains to show that every term in this resulting formula for $q$ is included in the information set $\mathcal{I}_1$ or can be computed from it without do-calculus. Then $\delta_q = (\mathcal{R}^\prime,\mathcal{P}^\prime)$ gives a do-calculus sequence for $q$. Let $q_m$ be the last element of the sequence $\mathcal{P}^\prime$. Any term in the transport formula $p_n$ that involves the target domain is unaffected by do-operators since no variable is available for experimentation in the target domain by Definition~\ref{def:so_into_tr}. Therefore, the corresponding term in $q_m$ can be obtained from $\mathcal{I}_1$ since this information set includes the full observed probability distribution $\P(v)$. Any term in $p_n$ that involves a source domain is necessarily affected by a do-operator, since the term would otherwise be identified from the target domain directly. Since Lemma~\ref{lem:dist_property} has already been applied, all such terms take the form $\P^{(i)}(c \cond \doo(z^\prime), d^\prime)$. Lemma~\ref{lem:dist_property} also guarantees, that the corresponding term $\P(c \cond \doo(z^\prime), d^\prime)$ in $q_m$ is computable from the information set $\mathcal{I}_1$.
\end{proof}

The proof of Theorem~\ref{thm:sufficient} provides a construction of a do-calculus sequence for a surrogate outcome identifiable causal effect through the query transformation. In practice, we do not have to retrace the entire derivation to obtain the identifying expression. It is enough to apply Lemma~\ref{lem:dist_property} to each relevant term in the resulting expression and them replace the terms in the transport formula with their respective counterparts from the information set of the surrogate outcome query. Appendix B contains examples on this process.
The following corollary describes the process of obtaining an expression for a surrogate outcome identifiable causal effect using the query transformation.

\begin{corollary} \label{cor:replacement}
Let $(X, Y, \mathcal{D}, G, \Pi, \pi^*, \mathcal{Z}, \emptyset )$ be the query transformation of a surrogate outcome query $(X, Y, G, \mathcal{S} )$ and suppose that there exists a transport formula $p_t$ for $\P^*(y \cond \doo(x))$ given by $\mathrm{TRSO}(y,x,\P^*(v),\emptyset,0,G,\mathcal{D},\varphi)$. Then $\P^*(y \cond \doo(x))$ is surrogate outcome identifiable and its expression is obtained from $p_t$ by manipulating every term in $p_t$ of the form $\P^{(i)}(c \cond \doo(z_i), d^\prime)$ in accordance to Lemma~\ref{lem:dist_property} and by omitting the domain indicators.
\end{corollary}
\begin{proof}
The result is a direct consequence of the construction for $\delta_q$ in the proof for Theorem~\ref{thm:sufficient}.
\end{proof}

We illustrate the application of TRSO and Corollary~\ref{cor:replacement} through an example. We use surrogate outcomes to identify $\P(y \cond \doo(x))$ in graph $G$ of Fig.~\ref{fig:theorem_example_main}(\subref{fig:theorem_exampleA}) from $\P(v)$ and $\P(w \cond \doo(x),a_1,a_2,b_1,b_2)$. By Definition~\ref{def:so_into_tr}, transportability nodes are added for $\De(X)_G \setminus \{W\} = \{X, B_1, B_2, Y\}$ and for the vertices in the same c-component as $W$ that are not ancestors of $W$ in $G[\overline{X}]$. The vertex $A_1$ is in the same c-component as $W$, but since it is still an ancestor of $W$ when edges incoming to $X$ are removed, no transportability node is added for it.

The transportability diagram of the corresponding query transformation is depicted in Fig.~\ref{fig:theorem_example_main}(\subref{fig:theorem_exampleB}). 

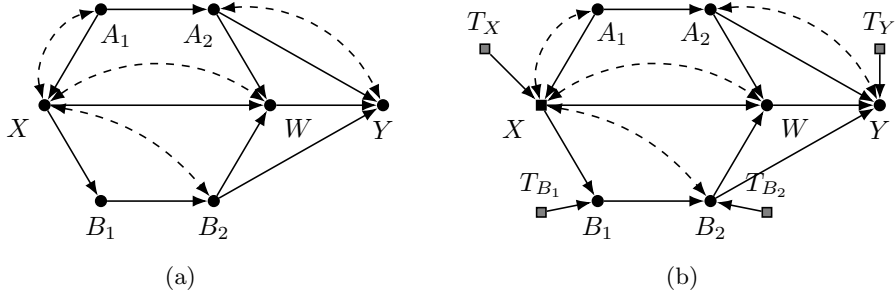
\begin{figure}[ht]
  \begin{subfigure}[t]{0.40\textwidth}
  \centering
  \begin{tikzpicture}[scale=1.5]
    \node [dot = {0}{0}{X}{below left}] at (0,0) {};
    \node [dot = {0}{0}{W}{below right}] at (2,0) {};
    \node [dot = {0}{0}{Y}{below}] at (3,0) {};
    \node [dot = {0.2}{0}{A_1}{below}] at (0.5,0.85) {};
    \node [dot = {-0.2}{0}{A_2}{below}] at (1.5,0.85) {};
    \node [dot = {0}{0}{B_1}{below}] at (0.5,-0.85) {};
    \node [dot = {0}{0}{B_2}{below}] at (1.5,-0.85) {};

    \draw [->] (X) -- (B_1);
    \draw [->] (X) -- (W);
    \draw [->] (W) -- (Y);
    \draw [->] (B_1) -- (B_2);
    \draw [->] (B_2) -- (W);
    \draw [->] (B_2) -- (Y);
    \draw [->] (A_1) -- (X);
    \draw [->] (A_1) -- (A_2);
    \draw [->] (A_2) -- (W);
    \draw [->] (A_2) -- (Y);

    \draw [<->,dashed] (X) to [bend left=45] (A_1);
    \draw [<->,dashed] (X) to [bend left=35] (W);
    \draw [<->,dashed] (X) to [bend left=20] (B_2);
    \draw [<->,dashed] (A_2) to [bend left=35] (Y);
  \end{tikzpicture}
  \caption{}
  \label{fig:theorem_exampleA}
  \end{subfigure}
  \hfill
  \begin{subfigure}[t]{0.50\textwidth}
  \centering
  \begin{tikzpicture}[scale=1.5]
    \node [intv = {0}{0}{X}{below left}] at (0,0) {};
    \node [dot = {0}{0}{W}{below right}] at (2,0) {};
    \node [dot = {0}{0}{Y}{below}] at (3,0) {};
    \node [dot = {-0.2}{0}{A_1}{below right}] at (0.5,0.85) {};
    \node [dot = {0.2}{0}{A_2}{below left}] at (1.5,0.85) {};
    \node [dot = {0}{0}{B_1}{below}] at (0.5,-0.85) {};
    \node [dot = {0}{0}{B_2}{below}] at (1.5,-0.85) {};
    \node [tr = {0}{0}{T_{B_1}}{above}] at (0,-0.95) {};
    \node [tr = {0}{0}{T_{B_2}}{above}] at (2,-0.95) {};
    \node [tr = {0}{0}{T_Y}{above}] at (3,0.5) {};
    \node [tr = {0}{0}{T_X}{above}] at (-0.5,0.5) {};

    \draw [->] (T_{B_1}) -- (B_1);
    \draw [->] (T_{B_2}) -- (B_2);
    \draw [->] (T_X) -- (X);
    \draw [->] (T_Y) -- (Y);

    \draw [->] (X) -- (B_1);
    \draw [->] (X) -- (W);
    \draw [->] (W) -- (Y);
    \draw [->] (B_1) -- (B_2);
    \draw [->] (B_2) -- (W);
    \draw [->] (B_2) -- (Y);
    \draw [->] (A_1) -- (X);
    \draw [->] (A_1) -- (A_2);
    \draw [->] (A_2) -- (W);
    \draw [->] (A_2) -- (Y);

    \draw [<->,dashed] (X) to [bend left=45] (A_1);
    \draw [<->,dashed] (X) to [bend left=35] (W);
    \draw [<->,dashed] (X) to [bend left=20] (B_2);
    \draw [<->,dashed] (A_2) to [bend left=35] (Y);
  \end{tikzpicture}
  \caption{}
  \label{fig:theorem_exampleB}
  \end{subfigure}
  \caption{Graphs for illustrating the application of TRSO and Corollary~\ref{cor:replacement}. (a) Graph $G$ corresponding to the surrogate outcome identifiability problem and the target domain of its query transformation. (b) Transportability diagram obtained from a query transformation corresponding to the domain where intervention on $X$ is available.}
  \label{fig:theorem_example_main}
\end{figure}
\noindent
The application of $\text{TRSO}$ succeeds in transporting the causal effect and produces the following expression for $\P^*(y \cond \doo(x))$
\begin{align*}
&\sum_{a_1,a_2,b_1,b_2,w}\!\! \P^*(y \cond a_1,x,a_2,b_1,b_2,w)\P^{(1)}(w \cond \doo(x), a_1,a_2,b_1,b_2)\P^*(a_2 \cond a_1,x) \\
&\quad \,\times\,\P^*(b_1 \cond a_1,x)\P^*(a_1)\sum_{a_1^\prime,x^\prime} \P^*(b_2 \cond a_1^\prime,x^\prime,b_1)\P^*(x^\prime \cond a_1^\prime)\P^*(a_1^\prime).
\intertext{\noindent
In this case we obtain the expression for the surrogate outcome identifiable causal effect $\P(y \cond \doo(x))$ by simply omitting all domain indicators from the expression for $\P^*(y \cond \doo(x))$ as licensed by Corollary~\ref{cor:replacement}
}
&\sum_{a_1,a_2,b_1,b_2,w}\!\! \P(y \cond a_1,x,a_2,b_1,b_2,w)\P(w \cond \doo(x), a_1,a_2,b_1,b_2)\P(a_2 \cond a_1,x) \\
&\quad \,\times\, \P(b_1 \cond a_1,x)\P(a_1)\sum_{a_1^\prime,x^\prime} \P(b_2 \cond a_1^\prime,x^\prime,b_1)\P(x^\prime \cond a_1^\prime)\P(a_1^\prime).
\end{align*}

\section{Discussion}
We take advantage of transportability by depicting experimental data as distinct domains and by using transportability nodes to prevent certain variables from being observed under an intervention. The positioning of the transportability nodes results in the need for Lemma~\ref{lem:dist_property} to parse the output of TRSO. It may be possible to consider other variations of the query transformation, where transportability nodes are omitted from additional vertices based on d-separation in the graph or by some other criteria. In the extreme case we could operate without any connection to transportability by omitting transportability nodes and relying on do-calculus entirely, but this approach can quickly become intractable for larger graphs. Our formulation avoids this, and the output can be directly transformed into a valid formula for a surrogate outcome identifiable causal effect. The query transformation has practical importance because an implementation of TRSO is readily available in the R package \emph{causaleffect}.

Transportability via the query transformation of Definition~\ref{def:so_into_tr} does not provide a complete characterization of surrogate outcome identifiability. As an example, we consider the graph of Fig.~\ref{fig:not_complete} and identifiability of $\P(y_1,y_2 \cond \doo(x_1,x_2))$ from $\P(v)$, $\P(y_2,z \cond \doo(x_2), x_1, y_1)$ and $\P(y_1 \cond \doo(x_2))$.

\begin{figure}[H]
  \centering
  \begin{tikzpicture}[scale=1.7]
    \node [dot = {0}{0}{X_1}{above}] at (1,0.75) {};
    \node [dot = {0}{0}{X_2}{above}] at (1,1.5) {};
    \node [dot = {0}{0}{Y_1}{below left}] at (0,0) {};
    \node [dot = {0}{0}{Y_2}{below right}] at (3,0) {};
    \node [dot = {0}{0}{Z}{below}] at (2,0) {};
    
    \draw [->] (X_1) -- (Z);
    \draw [->] (Z) -- (Y_2);
    \draw [->] (Y_1) -- (X_1);
    \draw [->] (Y_1) -- (Z);
    \draw [->] (X_2) -- (Y_1);
    \draw [->] (X_2) -- (Z);
    
    \draw [->] (X_2) to [bend left=35] (Y_2);
    \draw [<->,dashed] (X_2) to [bend right=45] (Y_1);
    \draw [<->,dashed] (Y_1) to [bend right=45] (Z);
    \draw [<->,dashed] (Y_2) to [bend left=45] (Y_1);
  \end{tikzpicture}
  \caption{An example where the query transformation does not produce a transportable causal effect, but the causal effect of interest is surrogate outcome identifiable.}
  \label{fig:not_complete}
  \end{figure}
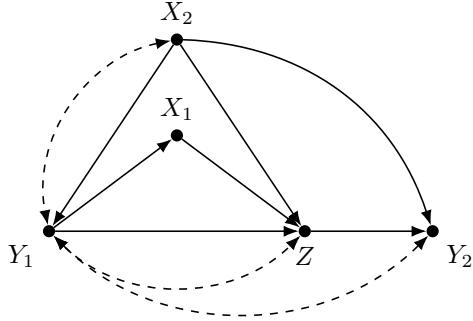
\noindent
We derive the effect using do-calculus:
\begin{align*}
P(y_1,y_2 \cond \doo(x_1,x_2)) &= \P(y_1 \cond \doo(x_1,x_2)) \P(y_2 \cond \doo(x_1,x_2),y_1) \\
                               &= \P(y_1 \cond \doo(x_2))  \P(y_2 \cond \doo(x_2),x_1,y_1) \\
                               &= \P(y_1 \cond \doo(x_2)) \sum_{z} \P(y_2,z \cond \doo(x_2),x_1,y_1),
\end{align*}
where the second equality follows from rules three and two by $(Y_1 \indep X_1)_{G[\overline{X_1},\overline{X_2}]}$ and $(Y_2 \indep X_1 \cond Y_1)_{G[\overline{X_2},\underline{X_1}]}$. It is easy to verify that the query transformation of this problem is not transportable using TRSO or the original transportability algorithm in \citep{bareinboim2014:completeness}. 

We performed a simple simulation study to assess the strength of TRSO. We generated 10000 instances of random graphs with 6 vertices and random sets of surrogate outcomes. For every instance, the causal effect  $p(y_1,y_2 \cond \doo(x_1,x_2))$ was verified to be non-identifiable from $P(v)$ alone using the ID algorithm. We used a simple exhaustive breadth-first forwards search that implements the rules of do-calculus and standard probability manipulations to confirm surrogate outcome identifiability or non-identifiability for each instance. Out of the 10000 instances 1514 were found to be surrogate outcome identifiable by the search and 1332 by TRSO which corresponds to 88 \% coverage. Based on this result, TRSO seems to be able to identify most of the surrogate outcome identifiable instances.

\section*{Conflict of interest statement}
The authors declare that they have no conflict of interest.

\section*{Acknowledgments} This work belongs to the thematic research area ``Decision analytics utilizing causal models and multiobjective optimization'' (DEMO) supported by Academy of Finland (grant number 311877). We thank the anonymous reviewers for their comments which helped to substantially improve this paper.

\section*{Appendix A}
\begin{proof}[Proof of Lemma~\ref{lem:dist_property}]
Let $G$ denote the graph of the original recursive call and assume without loss of generality that $G = G[An(Y)_G]$ for clarity. Let $D$ denote the graph of the current recursion stage. Consider a term of the form $\P^{(i)}(c \cond \doo(z^\prime), d^\prime)$ that appears in the output formula. Only line 6 of TRSO introduces permanent interventions into the expression by using the available experiments ($I = Z_i \cap X$), so it must have been triggered and it cannot be triggered again in the same recursive branch since we check that $I = \emptyset$ on this line. Before triggering line 6, only a combination lines 2, 3 and 4 can be triggered, corresponding to removal of non-ancestors of $Y$, introducing additional interventions via the third rule of do-calculus, and performing the c-component factorization, respectively. It follows that after these steps, the local distribution $P$ of one the recursive calls after triggering lines 2 and 4 in sequence is of the form $\P(\An(b)_G)$ where $B$ is the vertex set of some c-component of $G[V \setminus X]$

Line 6 is triggered next, activating an available experiment which means that there are no transportability nodes incoming to $B$ in $G$, since $(B \indep T_i \cond X)_D^{(i)}$. After this call and application of line 2, the local distribution is now of the form 
\[
\P^{(i)}(\An(b)_{G[\overline{Z^\prime}]} \setminus z^\prime \cond \doo(z^\prime)),
\]
where $Z^\prime = X \cap Z_i$ is the now active intervention. Since the sets $X$ and $Y$ local to this recursive call partition $V$ and non-ancestors have been removed, it is only possible to trigger line 1, 9 or 10 next, since we know that this call does not fail. We proceed to prove each case.

\textbf{Case of line 1.} Then the intervention set $X$ is empty and the effect is identified as
\[
\sum_{V \setminus B} \P^{(i)}(\An(b)_{G[\overline{Z^\prime}]} \setminus z^\prime \cond \doo(z^\prime)) = \P^{(i)}(b \cond \doo(z^\prime)).
\]
However, since $X$ and $Y$ local to this call still partition $V$ and $V = \An(B)_{G[\overline{Z}]} \setminus Z^\prime$ and $X$ is an empty set, we have that
\[
\P^{(i)}(b \cond \doo(z^\prime)) = \P^{(i)}(\An(b)_{G[\overline{Z^\prime}]} \setminus z^\prime \cond \doo(z^\prime)),
\]
meaning that $B$ contains its own ancestors in $G[\overline{Z^\prime}]$. If there exists a intervention--outcome pair $(Z_i,W_i) \in \mathcal{S}$ such that $W_i = B$, then the set $\An(w_i)_{G[\overline{Z^\prime}]} \setminus (w_i \cup z^\prime)$ is empty and we have that
\[
\P^{(i)}(b \cond \doo(z^\prime)) = \P^{(i)}(w_i \cond \doo(z^\prime), \An(w_i)_{G[\overline{Z^\prime}]} \setminus (w_i \cup z^\prime)),
\]
which corresponds to (3). When the domain indicator is omitted from this term, it is clearly identifiable from $\mathcal{I}$.

If instead there exists a intervention--outcome pair $(Z_i,W_i) \in \mathcal{S}$ such that for a subset $W^* \subset W_i$ it holds that $W^* \subset B$, then
\[
\P^{(i)}(b \cond \doo(z^\prime)) = \P^{(i)}(w^* \cond \doo(z^\prime), b \setminus w^*) \P^{(i)}(b \setminus w^* \cond \doo(z^\prime)).
\]
Since line 6 was triggered previously, $B$ cannot have any incoming transportability nodes, which means that $B \setminus W^*$ must be an ancestor of $W^*$ in $G[\overline{Z_i}]$ but not a descendant of $Z_i$ according to the construction of the transportability diagrams of a query transformation in Definition~\ref{def:so_into_tr}. This means that $B \setminus W^*$ is an ancestor of $W^*$ also in $G[\overline{Z^\prime}]$ and we have that
\[
  \P^{(i)}(w^* \cond \doo(z^\prime), b \setminus w^*) \P^{(i)}(b \setminus w^* \cond \doo(z^\prime)) = 
  \P^{(i)}(w^* \cond \doo(z^\prime), b \setminus w^*) \P^{(i)}(b \setminus w^*),
\]
which follows from the third rule of do-calculus since we have established that $B \setminus W_i^*$ must be a non-descendant of $Z^\prime$. Furthermore, since $B \setminus W^*$ can only contain ancestors of $W^*$ in $G[\overline Z]$ it follows that
\begin{equation} \label{eq:cond_identify}
  \P^{(i)}(w^* \cond \doo(z^\prime), b \setminus w^*) = \P^{(i)}(w^* \cond \doo(z^\prime), \An(w^*)_{G[\overline{Z^\prime}]} \setminus (w^* \cup z^\prime)).
\end{equation}
Now, we obtain from Definition~\ref{def:so_query} that 
\begin{align*}
\An(W^*)_{G[\overline{Z^\prime}]} \setminus W^* &= (\An(W_i)_{G[\overline{Z^\prime}]} \setminus W_i) \cup (\An(W^*)_{G[\overline{Z^\prime}]} \cap(W_i \setminus W^*)) \\
                                                &= (\An(W_i)_{G[\overline{Z^\prime}]} \setminus W_i) \cup (W_i \setminus W^*).
\end{align*}
This means that the right-hand side of \eqref{eq:cond_identify} can be obtained via conditioning by writing
\[
\frac{\P^{(i)}(w_i \cond \doo(z^\prime), \An(w_i)_{G[\overline{Z^\prime}]} \setminus (w_i \cup z^\prime))}{\sum_{W_i \setminus W^*} \P^{(i)}(w_i \cond \doo(z^\prime), \An(w_i)_{G[\overline{Z^\prime}]} \setminus (w_i \cup z^\prime))},
\]
which is identifiable from $\mathcal{I}$ after omitting domain indicators, which means that \eqref{eq:cond_identify} is identifiable as well. Therefore this case with $W^* \subset W_i$ corresponds to (1)

If instead there is no such $W_i^*$ we have
\[
  \P^{(i)}(b \cond \doo(z^\prime)) = \P^{(i)}(b),
\]
since now it must be the case that every member of $B$ is a non-descendant of $Z^\prime$ by Definition~\ref{def:so_into_tr}. This corresponds to option (2), and since $P(v)$ is always available, the term is identifiable from $\mathcal{I}$ after the omission of domain indicators. 

\textbf{Case of line 9.} In this case, the effect is identified as a conditional distribution 
\begin{align*}
&\sum_{B \setminus Y} \prod_{V_j \in B} \frac{\sum_{V \setminus V_\varphi^{(j)}}\P^{(i)}(\An(b)_{G[\overline{Z^\prime}]} \setminus z^\prime \cond \doo(z^\prime))}{\sum_{V \setminus V_\varphi^{(j-1)}}\P^{(i)}(\An(b)_{G[\overline{Z^\prime}]} \setminus z^\prime \cond \doo(z^\prime))} \\
=& \sum_{B \setminus Y} \prod_{V_j \in B} \P^{(i)}(v_j \cond \doo(z^\prime), \An(v_j)_{G[\overline{Z^\prime}]} \setminus (v_j \cup z^\prime)).
\end{align*}
As in the case of line 1, if there exists a $W_i^* \subset B$ such that $W_i^* \subset W_i$ and $(Z_i,W_i) \in \mathcal{S}$, then the product inside the sum takes the form
\begin{align} \label{eq:productform}
 & \prod_{V_j \in B} \P^{(i)}(v_j \cond \doo(z^\prime), \An(v_j)_{G[\overline{Z^\prime}]} \setminus (v_j \cup z^\prime)) \nonumber \\
=& \prod_{V_j \in W_i^*} \P^{(i)}(v_j \cond \doo(z^\prime), \An(v_j)_{G[\overline{Z^\prime}]} \setminus (v_j \cup z^\prime)) \nonumber \\
 & \quad \times\,\prod_{V_j \in (B \setminus W_i^*)} \P^{(i)}(v_j \cond \An(v_j)_{G[\overline{Z^\prime}]} \setminus (v_j \cup z^\prime)).
\end{align}
Here, individual terms of the form 
\[
\P^{(i)}(v_j \cond \doo(z^\prime), \An(v_j)_{G[\overline{Z^\prime}]} \setminus (v_j \cup z^\prime))
\]
are obtained via conditioning exactly as the right-hand side of \eqref{eq:cond_identify} by applying the same logic to $V_j$ instead of $W^*$ itself, which is valid for vertices $V_j \in W^*$.

The terms in the first product of \eqref{eq:productform} correspond to (3) and the terms in the second product correspond to (2). The equality again follows from the third rule of do-calculus that renders $V_j \in B \setminus W_i^*$ unaffected by the intervention on $Z^\prime$. If no suitable $W_i^*$ exists the product is simply
\[
\prod_{V_j \in B} \!\P^{(i)}(v_j \cond \doo(z^\prime), \An(v_j)_{D[\overline{Z^\prime}]} \setminus (v_j \cup z^\prime)) = \!\! \prod_{V_j \in B} \!\P^{(i)}(v_j \cond \An(v_j)_{D[\overline{Z^\prime}]} \setminus (v_j \cup z^\prime)),
\]
where the terms in the product correspond to (3) and the third rule of do-calculus is used again. These product terms are directly identifiable from $P(v)$ after omitting domain indicators.

\textbf{Case of line 10.} If line 10 was triggered with $I = \emptyset$ we are done, since the set of available experiments was set to $\emptyset$. If it was triggered with $I \neq \emptyset$, then we know that there are no incoming transportability nodes into the c-component consisting of the vertices $C^\prime$. It follows that the distribution $P$ of the next recursive call takes the form of \eqref{eq:productform} because the distribution of this call is $\prod_{v_i \in c} P(v_i \cond V_\varphi^{(i-1)} \cap C^\prime, v_\varphi^{(i-1)} \setminus c^\prime)$ and since the operations carried out on this distribution afterwards in the recursion can be represented by marginalization and conditioning by noting that on line 1, we return with $\sum_{v \setminus y} P$, on line 2 the recursive call contains $\sum_{V \setminus \An(Y)_D} P$, line 9 returns with $\sum_{C \setminus Y} \prod_{V_i \in C} (\sum_{V \setminus \varphi^{(i)}} P)/(\sum_{V \setminus V_\varphi^{(i-1)}} P)$, the recursive call on line 10 contains $\prod_{v_i \in c} P(v_i \cond V_\varphi^{(i-1)} \cap C^\prime, v_\varphi^{(i-1)} \setminus c^\prime)$ and $P$ remains unchanged in other recursive calls.

All the cases have been covered and the claim follows.
\end{proof}

\section*{Appendix B} This appendix contains examples on the construction of the do-calculus sequence in the proof of Theorem~\ref{thm:sufficient}. We begin with an example where we use surrogate outcomes to identify $p = \P(y \cond \doo(x))$ in the graph $G$ of Fig.~\ref{fig:theorem_example}(\subref{fig:theorem_exampleG}) from $\P(v)$ and $\P(z \cond \doo(x))$.
\begin{figure}[H]
  \begin{subfigure}[t]{0.5\textwidth}
  \centering
  \begin{tikzpicture}[scale=1.7]
    \node [dot = {0}{0}{W}{below}] at (0,0) {};
    \node [dot = {0}{0}{Z}{below}] at (1,0) {};
    \node [dot = {0}{0}{Y}{below}] at (2,0) {};
    \node [dot = {0}{0}{X}{above}] at (0.5,0.5) {};
    
    \draw [->] (W) -- (X);
    \draw [->] (W) -- (Z);
    \draw [->] (X) -- (Z);
    \draw [->] (Z) -- (Y);
    
    \draw [<->,dashed] (X) to [bend left=45] (Z);
  \end{tikzpicture}
  \caption{}
  \label{fig:theorem_exampleG}
  \end{subfigure}
  \begin{subfigure}[t]{0.45\textwidth}
  \centering
  \begin{tikzpicture}[scale=1.7]
    \node [dot = {0}{0}{W}{below}] at (0,0) {};
    \node [dot = {0}{0}{Z}{below}] at (1,0) {};
    \node [dot = {0}{0}{Y}{below}] at (2,0) {};
    \node [intv = {0}{0}{X}{above}] at (0.5,0.5) {};
    \node [tr = {0}{0}{T_Y}{above}] at (2,0.5) {};
    \node [tr = {0}{0}{T_X}{above}] at (0,0.5) {};
    
    \draw [->] (W) -- (X);
    \draw [->] (W) -- (Z);
    \draw [->] (X) -- (Z);
    \draw [->] (Z) -- (Y);
    \draw [->] (T_Y) -- (Y);
    \draw [->] (T_X) -- (X);
    
    \draw [<->,dashed] (X) to [bend left=45] (Z);
  \end{tikzpicture}
  \caption{}
  \label{fig:theorem_exampleD}
  \end{subfigure}
  \caption{Graphs related to the first example for illustrating Theorem 3.}
  \label{fig:theorem_example}
\end{figure}
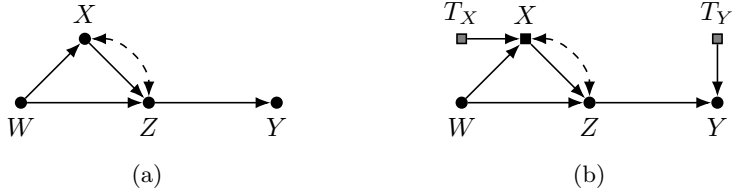
\noindent
Consider the surrogate outcome query $Q = (X,Y,G,\mathcal{S})$ where $X$ and $Y$ are the corresponding vertices of $G$ and $\mathcal{S} = \{(X,Z)\}$. The query transformation of $Q$ is $(X,Y,\{D\},G,\Pi,\pi^*,\{Z\}, \emptyset)$, where the transportability diagram $D$ depicted in Fig.~\ref{fig:theorem_example}(\subref{fig:theorem_exampleD}). By definition, transportability nodes are added for $X$ and $Y$, since $W$ is not a descendant of $X$ and $W$ is not in the same c-component as $Z$. We derive a do-calculus sequence that provides a transportability formula for the effect and construct a do-calculus sequence for the surrogate outcome identifiable causal effect 

\begin{figure}[H]
  \centering
  \footnotesize
  \begin{tabularx}{\textwidth}{@{\extracolsep{\fill}}lllr}
   $i$ & $\displaystyle p_i$ & & $R_i$ \\[5pt]
    0  & $\displaystyle \P^*(y \cond \doo(x))$ & & \\
    1  & $\displaystyle \sum_{z,w} \P^*(y,z,w \cond \doo(x))$ & & $m$ \\
    2  & $\displaystyle \sum_{z,w} \P^*(y \cond \doo(x), z,w)\P^*(z,w \cond \doo(x))$ & & $r$ \\
    3  & $\displaystyle \sum_{z,w} \P^*(y \cond \doo(x), z,w)\P^*(z \cond \doo(x), w)\P^*(w \cond \doo(x))$ & & $r$ \\
    4  & $\displaystyle \sum_{z,w} \P^*(y \cond \doo(x), z,w)\P^*(z \cond \doo(x), w,t_X,t_Y)\P^*(w \cond \doo(x))$ & & $(Z,\{T_X, T_Y\},X,W,1)$ \\
    5  & $\displaystyle \sum_{z,w} \P^*(y \cond x,z,w)\P^{(1)}(z \cond \doo(x), w)\P^*(w \cond \doo(x))$ & & $(Y,X,\emptyset,\{Z,W\},2)$ \\
    6  & $\displaystyle \sum_{z,w} \P^*(y \cond x,z,w)\P^{(1)}(z \cond \doo(x), w)\P^*(w)$ & & $(W,X,\emptyset,\emptyset,3)$ \\
       & & & \\
    $i$ & $\displaystyle p^\prime_i$ & & $R^\prime_i$ \\[5pt]
    0  & $\displaystyle \P(y \cond \doo(x))$ & & \\
    1  & $\displaystyle \sum_{z,w} \P(y,z,w \cond \doo(x))$ & & $m$ \\
    2  & $\displaystyle \sum_{z,w} \P(y \cond \doo(x), z,w)\P(z,w \cond \doo(x))$ & & $r$ \\
    3  & $\displaystyle \sum_{z,w} \P(y \cond \doo(x), z,w)\P(z \cond \doo(x), w)\P(w \cond \doo(x))$ & & $r$ \\
    4  & $\displaystyle \sum_{z,w} \P(y \cond x,z,w)\P(z \cond \doo(x), w)\P(w \cond \doo(x))$ & & $(Y,X,\emptyset,\{Z,W\},2)$ \\
    5  & $\displaystyle \sum_{z,w} \P(y \cond x,z,w)\P(z \cond \doo(x), w)\P(w)$ & & $(W,X,\emptyset,\emptyset,3)$
  \end{tabularx}
  \caption{A do-calculus sequence $\delta_p$ for $p = \P^*(y \cond \doo(x))$ for the graph of Fig.~\ref{fig:theorem_example}(\subref{fig:theorem_exampleG}) and the do-calculus sequence $\delta_q$ for $q = \P(y \cond \doo(x))$ for the graph of Fig.~\ref{fig:theorem_example}(\subref{fig:theorem_exampleD}) constructed as in the proof for Theorem~\ref{thm:sufficient}}
  \label{fig:do_calc_seq}
\end{figure}
\noindent
When a do-calculus sequence is considered, it is implicitly assumed that whenever $R_i \in \{m,r,p\}$ it is clear from the context which term or terms in the expression are referenced by the corresponding operation. In reality, these operations are more involved, for example marginalization should describe which term is being marginalized and which variables the operations is performed over. Similarly, $R_i$ corresponding to do-calculus manipulations reference the specific terms that are being manipulated. These details are omitted from the paper for clarity, since they are not crucial for the proofs and can impede readability. Figure~\ref{fig:do_calc_seq} shows the do-calculus sequences for the transportability query and the surrogate outcome query. The step transforming $p_3$ into $p_4$ is omitted from the do-calculus sequence for the surrogate outcome query according to the construction in Theorem 3. 

A second example highlights the omission of operations involving transportability diagrams. We use surrogate outcomes to identify $p = \P(y \cond \doo(x))$ in the graph $G$ of Fig.~\ref{fig:theorem_example2}(\subref{fig:theorem_exampleG2}) from $\P(v)$ and $\P(y \cond \doo(x),w_1,w_2)$.

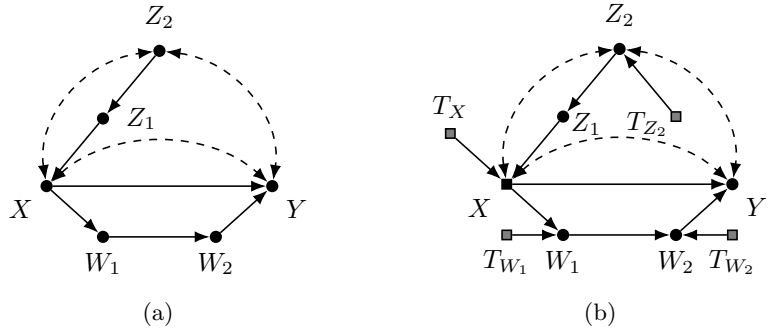
\begin{figure}[H]
  \begin{subfigure}[t]{0.5\textwidth}
  \centering
  \begin{tikzpicture}[yscale=0.9]
    \node [dot = {0}{0}{X}{below left}] at (0,0) {};
    \node [dot = {0.1}{0}{Z_1}{right}] at (0.75,1) {};
    \node [dot = {0}{0.1}{Z_2}{above}] at (1.5,2) {};
    \node [dot = {0}{0}{Y}{below right}] at (3,0) {};
    \node [dot = {0}{0}{W_1}{below}] at (0.75,-0.75) {};
    \node [dot = {0}{0}{W_2}{below}] at (2.25,-0.75) {};

    \draw [->] (X)   -- (Y);
    \draw [->] (Z_2) -- (Z_1);
    \draw [->] (Z_1) -- (X);
    \draw [->] (X) -- (W_1);
    \draw [->] (W_1) -- (W_2);
    \draw [->] (W_2) -- (Y);

    \draw [<->,dashed] (X)   to [bend left=45]  (Z_2);
    \draw [<->,dashed] (Z_2) to [bend left=45]  (Y);
    \draw [<->,dashed] (X) to [bend left=45]  (Y);
  \end{tikzpicture}
  \caption{}
  \label{fig:theorem_exampleG2}
  \end{subfigure}
  \begin{subfigure}[t]{0.45\textwidth}
  \centering
  \begin{tikzpicture}[yscale=0.9]
    \node [intv = {0}{0}{X}{below left}] at (0,0) {};
    \node [dot = {-0.1}{-0.1}{Z_1}{right}] at (0.75,1) {};
    \node [dot = {0}{0.1}{Z_2}{above}] at (1.5,2) {};
    \node [dot = {0}{0}{Y}{below right}] at (3,0) {};
    \node [dot = {0}{0}{W_1}{below}] at (0.75,-0.75) {};
    \node [dot = {0}{0}{W_2}{below}] at (2.25,-0.75) {};

    \node [tr = {0}{0}{T_{W_1}}{below}] at (0,-0.75) {};
    \node [tr = {0}{0}{T_{W_2}}{below}] at (3,-0.75) {};
    \node [tr = {0.1}{-0.1}{T_{Z_2}}{left}] at (2.25,1) {};
    \node [tr = {0}{0}{T_X}{above}] at (-0.75,0.75) {};

    \draw [->] (X)   -- (Y);
    \draw [->] (Z_2) -- (Z_1);
    \draw [->] (Z_1) -- (X);
    \draw [->] (X) -- (W_1);
    \draw [->] (W_1) -- (W_2);
    \draw [->] (W_2) -- (Y);

    \draw [->] (T_{W_1}) -- (W_1);
    \draw [->] (T_{W_2}) -- (W_2);
    \draw [->] (T_{Z_2}) -- (Z_2);
    \draw [->] (T_X) -- (X);

    \draw [<->,dashed] (X)   to [bend left=45]  (Z_2);
    \draw [<->,dashed] (Z_2) to [bend left=45]  (Y);
    \draw [<->,dashed] (X) to [bend left=45]  (Y);
  \end{tikzpicture}
  \caption{}
  \label{fig:theorem_exampleD2}
  \end{subfigure}
  \caption{Graphs related to the second example for illustrating Theorem 3.}
  \label{fig:theorem_example2}
\end{figure}
\noindent

From the derivation in Fig.~\ref{fig:do_calc_seq2} we can see that in order to add the necessary transportability nodes, we first have to manipulate the interventions present in the expression. We add the interventions for $Z_1$ and $Z_2$, which are later removed when they are no longer needed. These operations are reflected in the do-calculus sequence for the surrogate outcome query, even though adding the interventions is not necessary in this case. Despite of this fact, the sequence is valid.

\begin{figure}[H]
\vspace{5cm}
\begin{sideways}
  \footnotesize
  \begin{tabularx}{\textwidth}{@{\extracolsep{\fill}}lllr}
   $i$ & $\displaystyle p_i$ & & $R_i$ \\[5pt]
    0  & $\displaystyle \P^*(y \cond \doo(x))$ & & \\
    1  & $\displaystyle \P^*(y \cond \doo(x,z_1,z_2))$ & & $(Y,\{Z_1,Z_2\},X,\emptyset,3)$ \\
    2  & $\displaystyle \sum_{w_1,w_2} \P^*(y,w_1,w_2 \cond \doo(x,z_1,z_2))$ & & $m$\\
    3  & $\displaystyle \sum_{w_1,w_2} \P^*(y \cond \doo(x,z_1,z_2), w_1,w_2) \P(w_1,w_2\cond \doo(x,z_1,z_2))$ & & $r$\\
    4  & $\displaystyle \sum_{w_1,w_2} \P^*(y \cond \doo(x,z_1,z_2), w_1,w_2,t_{W_1},t_{W_2},t_X,t_{Z_2}) \P(w_1,w_2\cond \doo(x,z_1,z_2))$ & & $(Y, \{T_{W_1},T_{W_2},T_X,T_{Z_2}\}, \{X,Z_1,Z_2\},\emptyset,1)$ \\
    5  & $\displaystyle \sum_{w_1,w_2} \P^{(i)}(y \cond \doo(x,z_1,z_2), w_1,w_2) \P(w_1,w_2\cond \doo(x))$ & & $(\{W_1,W_2\}, \{Z_1,Z_2\}, X,\emptyset,3)$ \\
    6  & $\displaystyle \sum_{w_1,w_2} \P^{(i)}(y \cond \doo(x,z_1,z_2), w_1,w_2) \P(w_1,w_2\cond x)$ & & $(\{W_1,W_2\}, X, \emptyset,\emptyset,2)$ \\
    7  & $\displaystyle \sum_{w_1,w_2} \P^{(i)}(y \cond \doo(x), w_1,w_2) \P(w_1,w_2\cond x)$ & & $(Y, \{Z_1,Z_2\}, \{X\},\emptyset, 3)$ \\
       & & & \\
   $i$ & $\displaystyle p_i$ & & $R_i$ \\[5pt]
    0  & $\displaystyle \P(y \cond \doo(x))$ & & \\
    1  & $\displaystyle \P(y \cond \doo(x,z_1,z_2))$ & & $(Y,\{Z_1,Z_2\},X,\emptyset,3)$ \\
    2  & $\displaystyle \sum_{w_1,w_2} \P(y,w_1,w_2 \cond \doo(x,z_1,z_2))$ & & $m$\\
    3  & $\displaystyle \sum_{w_1,w_2} \P(y \cond \doo(x,z_1,z_2), w_1,w_2) \P(w_1,w_2\cond \doo(x,z_1,z_2))$ & & $r$\\
    4  & $\displaystyle \sum_{w_1,w_2} \P(y \cond \doo(x,z_1,z_2), w_1,w_2) \P(w_1,w_2\cond \doo(x))$ & & $(\{W_1,W_2\}, \{Z_1,Z_2\}, X,\emptyset,3)$ \\
    5  & $\displaystyle \sum_{w_1,w_2} \P(y \cond \doo(x,z_1,z_2), w_1,w_2) \P(w_1,w_2\cond x)$ & & $(\{W_1,W_2\}, X, \emptyset,\emptyset,2)$ \\
    6  & $\displaystyle \sum_{w_1,w_2} \P(y \cond \doo(x), w_1,w_2) \P(w_1,w_2\cond x)$ & & $(Y, \{Z_1,Z_2\}, \{X\},\emptyset, 3)$
  \end{tabularx}
\end{sideways}
  \caption{A do-calculus sequence $\delta_p$ for $p = \P^*(y \cond \doo(x))$ in the graph of Fig.~\ref{fig:theorem_example2}(\subref{fig:theorem_exampleG2}) and the do-calculus sequence $\delta_q$ for $q = \P(y \cond \doo(x))$ in the graph of Fig.~\ref{fig:theorem_example2}(\subref{fig:theorem_exampleD2}) constructed as in the proof for Theorem~\ref{thm:sufficient}.}
  \label{fig:do_calc_seq2}
\end{figure}


\section*{References}
\bibliographystyle{elsarticle-num}
\bibliography{experimental_id}
\end{document}